\documentclass[twoside]{article}
\usepackage[left=2cm,right=2cm,top=2cm,bottom=2cm]{geometry}
\usepackage{amsmath}
\usepackage{amsfonts}
\usepackage{amssymb}
\usepackage{graphicx}
\usepackage{color}
\usepackage[normalem]{ulem}
\usepackage[utf8]{inputenc}
\usepackage{amsfonts}
\usepackage{color}
\usepackage[normalem]{ulem}
\newenvironment{proof}{\noindent{\sc Proof.}}{\qed}
\newtheorem{theorem}{Theorem}[section]
\newtheorem{lemma}{Lemma}[section]
\newtheorem{cor}{Corollary}[section]
\newtheorem{rem}{Remark}[section]
\newtheorem{prop}{Proposition}[section]
\newtheorem{uda}{Example}[section]
\newcommand{\qed}{$\blacksquare$}

\def\bhag#1{\noindent
\setcounter{equation}{0}
\section{#1}
}

\def\HH{{\mathbb H}}

\def\RR{{\mathbb R}}

\def\ZZ{{\mathbb Z}}

\def\TT{\mathbb T}
\def\T{\mathcal{T}}

\def\bs#1{{\boldsymbol{#1}}}
\def\x{\mathbf{x}}
\def\k{\mathbf{k}}
\def\y{\mathbf{y}}
\def\u{\mathbf{u}}

\def\m{\mathfrak{m}}
\def\j{\mathbf{j}}

\def\t{\mathbf{t}}

\def\O{{\cal O}}

\def\C{{\mathcal C}}

\def\argmin{\mathop{\hbox{\textrm{arg min}}}}

\def\be{\begin{equation}}
\def\ee{\end{equation}}
\def\bea{\begin{eqnarray}}
\def\eea{\end{eqnarray}}
\def\eref#1{(\ref{#1})}
\def\disp{\displaystyle}

\def\donchitre#1#2{\vskip 6.5cm\noindent
\parbox[t]{1in}{\special{eps:#1.eps x=6.5cm y=5.5cm}}
\hbox to 7cm{}\parbox[t]{0.0cm}{\special{eps:#2.eps x=6.5cm y=5.5cm}}}

\def\XX{{\mathbb X}}
\def\BB{{\mathbb B}}

\def\bs#1{{\boldsymbol{#1}}}

\title{An analysis of training and generalization errors in shallow and deep networks}

\author{
 H.~N.~Mhaskar\thanks{
Institute of Mathematical Sciences, Claremont Graduate University, Claremont, CA 91711. The research of this author is supported in part by the Office of the Director of National Intelligence (ODNI), Intelligence Advanced Research Projects Activity (IARPA), via 2018-18032000002.
\textsf{email:} hrushikesh.mhaskar@cgu.edu} 
 \ and   T. Poggio \thanks{
   Center for Brains, Minds, and Machines, McGovern Institute for Brain Research,
   Massachusetts Institute of Technology, Cambridge, MA, 02139. The research of this author is supported by the Center for Brains, Minds and
  Machines (CBMM), funded by NSF STC award CCF-1231216.
 \textsf{tp@mit.edu} }
 }
 \date{}
\begin{document}

\maketitle

\begin{abstract}
This paper is motivated by an open problem around  deep networks, namely, the apparent absence of
over-fitting despite large 
over-parametrization which allows perfect
fitting of the training data. 
In this paper, we analyze  this phenomenon in the case of regression problems when each unit evaluates a periodic activation function. 
We argue that the minimal expected value of the square loss is inappropriate to measure the generalization error in approximation of compositional functions in order to take full advantage of the compositional structure. Instead, we measure the generalization error in the sense of maximum loss, and sometimes, as a pointwise error.
We give estimates on exactly how many parameters ensure both zero training error as well as a good generalization error. 
We prove that a solution of a regularization problem is guaranteed to yield a good training error as well as a good generalization error and estimate how much error to expect at which test data. 

\end{abstract}


\bhag{Introduction}\label{intsect}

The main problem of machine learning is the following. 
Given data $(\x,y)$ sampled from an unknown probability distribution $\mu$, the goal is to find a function $P$ that minimizes the generalization error $\mathbb{E}_\mu((y-P(\x))^2)$ among all functions in some function class that is thought to represent the prior information about the distribution. 
Since we do not know $\mu$, classical machine learning paradigm expresses the generalization error as a sum of three components: the variance, the approximation error, and the sampling error. 
The variance is a lower bound on the generalization error, and the estimation typically focuses on the other two errors. The sum of these two is given by $\mathbb{E}((f-P(\x))^2)$, where the expectation is with respect to the marginal distribution of $\x$ and the \emph{target function} $f$ is the conditional expectation of $y$ given $\x$.
 If the marginal distribution of $\x$ is known, then the split between approximation and sampling errors is no longer necessary, and one can obtain estimates  as well as constructions directly from characteristics of the training data (e.g., \cite{quadconst, eignet, eugenenevai}). 
In the classical paradigm where this distribution is not known,
the approximation error decreases as the number of parameters in $P$ increases to $\infty$.
However, this makes the empirical risk minimization problem harder to solve; making it essential to choose the number of parameters in $P$ to balance the two errors. 
In turn, this explains a commonly observed phenomenon that if one achieves a zero empirical risk on the training data by over-parametrized model $P$, the test error is generally not good.

There are several recent papers that demonstrate that this phenomenon is often not observed (e.g., \cite{hardt2016identity, neyshabur2017exploring, sokolic2016robust, zhangmusings, belkin2018understand, poggio2017theory3}).
There are a few recent papers that address this issue in the case of classification problems. Belkin, Hsu, and Mitra \cite{belkin2018overfitting} analyze the ``excess error'' in least square 
fits by piecewise linear interpolants over that obtained by the optimal Bayes' classifier. 
In \cite{poggio2017theory3, poggio2018theory3b}, the question is analyzed from the perspective of the geometry of the error surface with respect to different loss functions near the local extrema.
In particular, it is shown in \cite{poggio2017theory3} that substituting the ReLU activation function by a polynomial approximation exhibits the same behavior as the original network.

In this paper, we focus on regression problems. 
We wish to consider from the point of view of approximation theory the overfitting puzzle for universal approximation  in a more intrinsic and theoretical manner, independent of the training mechanism used. 
This paper does not seek to ``explain'' the overfitting phenomenon as observed. Such an explanation needs to  involve not just the analysis of the approximation problem itself, but also an analysis of the training algorithms used for this purpose. 
It is obvious that a network (or any other model)  using a number of parameters that is less  than the number  of training data points cannot in general produce a zero training error in the absence of some strong prior knowledge about the target function that generated the training data, no matter what training algorithm is used.
Our goal is to study the fundamental problem of function approximation to examine what characteristics of the data and how much overparametrization will give \textbf{theoretical guarantees} that the generalization error can be controlled while achieving a zero training error. 
We will do this without any prior assumption about the target function apart from continuity or at most differentiability. 

We wish to address the following issues about the  phenomenon of zero/small training error and small test error for universal approximation:
\begin{enumerate}
\item What characteristics of the data govern a zero training error and a good generalization error? 
\item How much over-parametrization will give a \textbf{mathematical guarantee} of the model to exhibit this behavior?
\item Propose a regularization scheme whose solution will guarantee a good (but not necessarily zero) training error while maintaining a good generalization error at the same time.
The emphasis here is on an estimation of  how much over-parametrization is necessary to get theoretical guarantees.
\item What bounds on the generalization error can be guaranteed by the solution of the (global) regularization scheme \textbf{at each point} in the test data (rather than a global error estimate), compared to the nearest neighbor in the training data? 
\end{enumerate}

In recent years, convolutional neural networks (CNNs) have achieved a revolution in machine learning. 
A good survey can be found in \cite{lecun2015deep}. From a practical point of view, the central features of CNNs are locality and weight sharing. 
From a strictly mathematical point of view, convolution is a very special operation that requires a group structure on the data.
According to the book \cite{goodfellow2016deep}, the CNNs ``are a specialized kind of neural network for processing data that has
 a known, grid-like topology. Examples include time-series data, which can be
 thought of as a 1D grid taking samples at regular time intervals, and image data,
 which can be thought of as a 2D grid of pixels.'' 
 For this kind of data, it is customary to treat it either as  data on the whole real line/plane with zero-padding, or otherwise use a symmetrical extension to treat it as  data on a circle/torus so that the standard group operations on these spaces can be used to define the operation of convolution.

In this paper, we will focus on function approximation on the torus. 
The most fundamental class for this purpose is the class of all trigonometric polynomials.
Accordingly, we will study the problem of the lack of overfitting in the context of approximation by trigonometric polynomials. 
We will explain in Section~\ref{trigtonetsect} the theoretical relationship between  trigonometric polynomials and neural/RBF networks in the periodic setting. In Section~\ref{shallowsect}, we will show how the theorems about trigonometric approximation translate into theorems about shallow networks with arbitrary periodic activation functions.

In the study of approximation error in deep learning, it is observed in \cite{dingxuanpap} that the compositional structure of the target function can be utilized effectively via a property called good propagation of error to obtain substantially better error bounds allowing us to overcome the curse of dimensionality observed in shallow networks.
This allows us to ``lift'' many results on approximation by shallow networks to those on deep networks.
However, there are a few barriers that prevent a straightforward extension.

One is that we do not know the functions evaluated at each node in the intervening layers; just the input/output relations between the data and output of the ultimate layers.  
There are some recent efforts in the direction of designing deep networks in some special applications using domain knowledge (e.g., \cite{han2019k,mhas_sergei_maryke_diabetes2017}).
It is conceivable that this problem does not appear in these contexts.

The other problem is even more fundamental, requiring a change in the traditional notion of generalization error.
For example, suppose we wish to approximate a function $f^*=f(f_1(\x_1), f_2(\x_2))$ by a network of the form $P^*=P(P_1(\x_1), P_2(\x_2))$, where $\x_1, \x_2\in \RR^d$. 
Without the compositional structure, $f^*$ has to be treated as a function of $2d$ variables.
The compositional structure allows us to treat the approximation problem as a set of three approximation problems: approximating functions $f_1$, $f_2$ as functions of $d$ variables each, and a function $f$ of $2$ variables. 

How do we define generalization error?
Defining it in terms of the original distribution of $((\x_1,\x_2),y)$ is not sensitive to the compositional structure. 
On the other hand, the input $(f_1,f_2)$ to the function $f$ is not the same (even may not have the same distribution) as the input $(P_1,P_2)$ to the approximation $P$.

Therefore, we measure in this paper the errors in the uniform norm rather than searching for appropriate $L^2$ norms suitable for the compositionality structure of the target function. 
Thus, we define the generalization error as the maximal error  between the target function and the model at all possible test points. 
One consequence is that the decomposition of the generalization error into three parts breaks down. Therefore, new ideas are required to achieve the approximation in terms of the training data alone.
The approximation errors themselves are studied in \cite{dingxuanpap, sphrelu}, but the techniques suggested there require the data points $\x_j$ to be sufficiently dense in the domain (Euclidean space, sphere, cube, etc.).
When the data is not dense, it is clearly not expected to get a good approximation on the whole domain. 
However, if the data does become denser and denser on some subset of the domain, one can expect a good approximation at points close by to the training data. Thus, we will establish pointwise bounds for the generalization error obtained by a solution of a  regularization scheme suggested for this purpose.

In Section~\ref{backsect}, we develop some notation and provide some background information that has motivated our current paper.
In Section~\ref{shallowsect}, we state our theorems for the case of shallow networks. In Section~\ref{deepsect}, we state  the analogues of the results in Section~\ref{shallowsect} in the case of deep networks. The proofs of the results in Sections~\ref{shallowsect} and \ref{deepsect} are given in Section~\ref{pfsect}.

\bhag{Background}\label{backsect}

The purpose of this section is to explain the connection between trigonometric polynomials and neural networks (Section~\ref{trigtonetsect}) as well as to explain a classical theorem which provides a motivation for our theorems in this paper (Section~\ref{intapproxsect}). In order to do so, we need first to develop some notation. This is done in Section~\ref{notationsect}.

\subsection{Notation}\label{notationsect}

Let $q\ge 1$ be an integer, $\TT^q$ be the $q$ dimensional torus (=$\RR^q/(2\pi\ZZ)^q$). For $\x, \y\in\TT^q$, 
$$
|\x-\y|=\max_{1\le i\le q}|(x_i-y_i)(\mbox{ mod } 2\pi)|.
$$
For a multi-integer $\k$, $|\k|_p$ is the $\ell^p$ norm of $\k$. 

The space  of all continuous functions $f:\TT^q\to\RR$, equipped with the supremum norm will be denoted by $C^*$ (or $C^*(\TT^q)$ if we wish to emphasize the input dimension to the functions). The norm on $C^*(\TT^q)$ will be denoted by $\|\cdot\|$ or $\|\cdot\|_q$ if it is important to identify the dimension.
For $n>0$, the space $\HH_n^q$ of trignometric polynomials in $q$ variables with (spherical) degree $<n$ is defined by
$$
\HH_n^q=\mathsf{span}\{\exp(i\k\cdot\circ) : |\k|_2 <n\}.
$$
The dimension of $\HH_n^q$ is $\sim n^q$.
If $f\in C^*(\TT^q)$, then its 
Fourier coefficients are defined by 
\be\label{fourcoeffdef}
\hat{f}(\k)=\frac{1}{(2\pi)^q}\int_{\TT^q}f(\x)\exp(-i\k\cdot\x)d\x, \qquad \k\in\ZZ^q,
\ee
and its
degree of approximation from
$\HH_n^q$ is defined by
$$
E_n(f)=E_{n}(q;f):=\inf_{T\in\HH_n^q}\|f-T\|.
$$
When our models are trigonometric polynomials in $\HH_n^q$, the quantity $E_n(f)$ denotes the ideal generalization error for the target function $f$. 

Let $h:\RR\to [0,1]$ be an infinitely differentiable, even function such that $h(t)=1$ if $|t|\le 1/2$, $h(t)=0$ if $|t|\ge 1$. We define 
\be\label{kerndef}
\Phi_N(\x)=\sum_{\k \in \ZZ^q}h\left(\frac{|\k|_2}{N}\right)\exp(i\k\cdot\x), \qquad \x\in\TT^q,\ N>0.
\ee

For $f\in C^*$, we define
\be\label{summkerndef}
\sigma_N(f)(\x)=\sum_{\k\in\ZZ^q}h\left(\frac{|\k|_2}{N}\right)\hat{f}(\k)\exp(i\k\cdot\x), \qquad \x\in\TT^q,\ N>0.
\ee
We note that the sums in both \eref{kerndef} and \eref{summkerndef} are finite sums, although they are written as infinite sums for  convenience of notation.

For any finite subset $\C\subset \TT^q$, we define its minimal separation by
\be\label{minsepdef}
\eta(\C)=\min_{\x,\y\in \C, \x\not=\y}|\x-\y|.
\ee
We assume a training data of the form $\mathcal{D}=\{(\x_j, y_j)\}_{j=1}^M$, where $\mathcal{C}=\{\x_j\}_{j=1}^M \subset \TT^q$, and $y_j=f(\x_j)+\epsilon_j$ for some $f\in C^*$. We denote
\be\label{epsilondef}
\epsilon=\max_{1\le j\le M} |\epsilon_j|.
\ee
The quantity $\epsilon$ plays the role of variance in our theory in this paper. 
In regression problems, where numerical accuracy is expected in the supremum norm, the errors $\epsilon_j$ all need to be small. 
Thus, even though the quantity $\epsilon$ seems to increase with $M$, it remains small as $M\to\infty$. 
One example is when the values of $f$ are not computed exactly (as they rarely are), but only through a numerical computation (for example, $f$ is the ideal solution of a differential equation, but the data is obtained by solving this equation numerically at the grid points). 
The quantity $\epsilon$ then represents the maximal error in this numerical computation.\\

\noindent{\textbf{The constant convention}\\
\emph{The
symbols $c, c_1,\cdots$ will denote generic positive constants,
depending on such fixed parameters of the problem as   $q$, $h$, $\mathcal{G}$, and $S$ (to be introduced later),
etc. and other quantities explicitly indicated, but their values may be
different at different occurrences, even within a single formula. The
notation $A\sim B$ means that $c_1A\le B\le c_2A$.}\\

\subsection{Neural networks and trigonometric polynomials}\label{trigtonetsect}
The material in this section is based on \cite{mhaskar1995degree}.
For reasons that will become clear shortly, the term \emph{activation function} in this paper will mean $\phi\in C^*$ such that $\hat{\phi}(1)\not=0$.
We note that a trigonometric polynomial is itself a neural network with the activation function $t\mapsto \cos t$. There is a close connection between trigonometric polynomials and networks with other activation functions.
 Let $\phi \in C^*(\TT)$ and $\hat{\phi}(1)\not=0$. Then for $\k\in\ZZ^q$ and $\x\in\TT^q$, it is not difficult to verify that
$$
\exp(i\k\cdot\x)=\frac{1}{2\pi\hat{\phi}(1)}\int_\TT \phi(t)\exp\left(i(\k\cdot\x-t)\right)dt=\frac{1}{2\pi\hat{\phi}(1)}\int_\TT \phi(\k\cdot\x-t)\exp(it)dt.
$$
Discretizing the integral expression above, it can be shown (cf. \cite[Proposition~4.2.1]{indiapap}) that for any integer $N\ge 1$, $\k\in\ZZ^q$,
\be\label{trigasneunet}
\left\|\exp(i\k\cdot(\circ))- \frac{1}{(2N+1)\hat{\phi}(1)}\sum_{j=0}^{2N}\exp\left(\frac{2\pi i j}{2N+1}\right)\phi\left(\k\cdot(\circ)-\frac{2\pi  j}{2N+1}\right)\right\| \le \frac{4}{|\hat{\phi}(1)|}E_N(1; \phi).
\ee
In particular, if 
 $T$ is a trigonometric polynomial,  let
\be\label{trigneunetdef}
\mathbb{G}_N(\phi, T)(\x)= \frac{1}{(2N+1)\hat{\phi}(1)}\sum_{j=0}^{2N}\exp\left(\frac{2\pi i j}{2N+1}\right)\left(\sum_{\k\in\ZZ^q} \hat{T}(\k)\phi\left(\k\cdot(\circ)-\frac{2\pi  j}{2N+1}\right)\right),
\ee
where it is understood that $\hat{T}(\k)=0$ if $|\k|_2$ exceeds the degree $n$ of $T$. 
The number of neurons involved in the network 
$\mathbb{G}_N(\phi,T)$ is $\sim Nn^q$.
The estimate \eref{trigasneunet} leads to
\be\label{trigneunetapprox}
\left\|T-\mathbb{G}_N(\phi,T)\right\|
\le \frac{4}{|\hat{\phi}(1)|}E_N(1; \phi)\sum_{\k\in\ZZ^q}|\hat{T}(\k)|.
\ee
Thus, a trigonometric polynomial can be approximated by a neural network with activation function $\phi$ and the error bounds can be 
obtained by keeping track of the degree of approximation of the target function by trigonometric polynomials, the magnitude of its Fourier coefficients and the bound in \eref{trigneunetapprox} (cf. \cite{mhaskar1995degree} for some examples). This leads to the following proposition, which will be proved in Section~\ref{pfsect}.

\begin{prop}\label{periodic_relu_prop}
Let $\phi\in C^*(\TT)$, $\hat{\phi}(1)\not=0$, $f\in C^*$.  Then for $n,N\ge 1$, we have
\be\label{relu_approx}
\begin{aligned}
\|f-\mathbb{G}_N(\phi,\sigma_n(f))\|&=\left\|f  -\frac{1}{(2N+1)\hat{\phi}(1)}\sum_{j=0}^{2N}\exp\left(\frac{2\pi i j}{2N+1}\right) \left(\sum_{|\k|_2<n} h\left(\frac{|\k|_2}{n}\right)\hat{f}(\k)\phi\left(\k\cdot(\circ)-\frac{2\pi  j}{2N+1}\right)\right)\right\|\\
&\le c(\phi)\left\{E_{n/2}(q;f)+n^{q/2}E_N(1;\phi)\|f\|\right\}.
\end{aligned}
\ee
\end{prop}

\begin{uda}\label{periodic_relu_rmk}
{\rm For example, we consider  the smooth ReLU function $t\mapsto \log(1+e^t)=t_++\O(e^{-|t|})$. Then the function
$\disp\psi(t)=\log\left(\frac{(1+e^{t+\pi})(1+e^{t-\pi})}{(1+e^t)^2}\right)$
is integrable on $\RR$. 
The periodization 
\be\label{smoothreludef}
\phi(t)=\sum_{j\in\ZZ}\psi(t+2\pi j), \qquad t\in\RR,
\ee
is an analytic function on $\TT$. So, Bernstein approximation theorem \cite[Theorem~5.4.2]{timanbk} shows that there exists $\rho_1<1$ with $E_N(1;\phi)\le \rho_1^N$ for all $N\ge c(\phi)$. 
In Proposition~\ref{periodic_relu_prop}, if $f$ satisfies
$E_n(q;f)=\O(n^{-\gamma})$ for some $\gamma$, then we may choose $N\sim \log n$ to get a network with $\O(n^q\log n)$ neurons 
to obtain an estimate $\O(n^{-\gamma})$ on the right hand side of \eref{relu_approx}.
\qed}
\end{uda}

This idea is generalized to many other settings, and algorithms are known to find the approximation to the target function using the training data, \textbf{without assuming any prior on the target function} (see, e.g. \cite{indiapap} for an early construction). 
However, formulating the problem directly as a minimization  of the supremum norm error between the function and the neural network model may not work. 
The theory implies certain relationships between the coefficients and the weights and thresholds.

Conversely, one can  approximate $\phi$ by trigonometric polynomials. 
Therefore, if one can obtain or assume some bounds on the coefficients of a neural network with $\phi$ as the activation function, then these bounds can be translated to bounds on the degree of approximation by trigonometric polynomials.
This part is hard to do on the torus with neural networks, but has been done in a far more general setting with kernel based approximation \cite{eignet}, where Mercer expansions satisfying certain technical conditions are known.

In view of this close relationship between general neural networks and trigonometric polynomials (i.e., networks with activation function $t\mapsto\cos t$), we will focus in this paper on trigonometric polynomials, and demonstrate in Section~\ref{shallowsect} how these results are translated to those with other neural networks.

\subsection{Interpolatory approximation}\label{intapproxsect}

In the language of  classical approximation theory, the problem of achieving a zero training error is the problem of interpolation.
In the context of trigonometric polynomials,  for any data of the form $\{(\theta_j, y_j)\}_{j=0}^{2n}$, $y_j\in\RR$, $\theta_j\in \TT$, and $\theta_j\not=\theta_k$ if $j\not=k$, there exists $T\in\HH_n^1$ such that $T(\theta_j)=y_j$ for $j=0,\cdots, 2n$ \cite[Chapter~X, Theorem~1.2]{zygmund}. 
Thus, it is easy to obtain a zero training error with a minimal number of free parameters.
As we argued in the introduction, the test error in this context should be measured in terms of uniform approximation to the target function. 
A well known theorem attributed  in \cite{natanson} to Faber and Bernstein states that for \textbf{any} system of interpolation nodes, there exists a function $ f\in C^*(\TT)$ for which the sequence of interpolatory trigonometric polynomials (with minimal degree as above)  does not converge uniformly to $f$. 

In 1943, Erd\H os \cite{erdos1943some} initiated (in the context of algebraic polynomials) a study of the question whether one can get a convergent sequence of interpolatory polynomials if one allows a polynomial of higher than minimal degree. 
A positive answer was given by Szabados in \cite{szabados1978some}. 
Although the answer is given in terms of algebraic polynomials, Szabados remarks that the same is clearly true for trigonometric polynomials as well.
An explicit statement to this effect is the following \cite[Theorem~3.1(a)]{approxint2002}.

\begin{theorem}\label{approxinttheo}
Let $\theta_1,\ldots,\theta_N$ be 
distinct points in $[-\pi,\pi]$, $\theta_{N+1}:=\theta_1+2\pi$, 
$\alpha>0$, and 
$$
\min_{1\le k\le N}|\theta_{k+1}-\theta_k|=: \eta.
$$
Then for $f\in C^*(\TT^1)$, 
there exists a trigonometric polynomial $T$ of 
degree at most $\disp (1+2/\alpha)(\pi/\eta)$ such that 
$f(\theta_j)=T(\theta_j)$, $1\le j\le N$, and
\be\label{szabad_bd}
\|f-T\| \le (2+\alpha)E_m(1;f)
\ee
where $\disp m= (1+2/\alpha)(\pi/\eta)$.
\end{theorem}

\begin{rem}\label{density_rmk}
{\rm A curious feature of Theorem~\ref{approxinttheo} is that one obtains the bound \eref{szabad_bd} on the generalization on the entire torus $\TT^1$ without requiring that the training data $\C$ be ``dense'' in $\TT^1$. 
}
\end{rem}
\bhag{Shallow networks}\label{shallowsect}

Our first theorem in this section shows the connection between the structural properties of the training data and the construction of a trigonometric polynomial $T^\#_N(\mathcal{D})$ that can interpolate the noisy data (i.e., achieve a zero training error), as well as achieve a good generalization error in the sense defined in Section~\ref{intsect}.

Before stating our main results, we first discuss two straightforward ideas. The first of these is to construct a trigonometric interpolatory polynomial $\mathcal{I}(\mathcal{D})$ of minimal degree. The other is to use a larger degree $N$, and solve the system of equations (cf. \eref{kerndef})
\be\label{loc_int_eqns}
\sum_{j=1}^M a_j\Phi_N(\x_\ell-\x_j)=y_\ell, \qquad \ell=1,\cdots, M.
\ee
For a sufficiently large value of $N$, one can show that this sytem of equations has a unique and stable solution. 
We denote the corresponding  polynomial by \be\label{locintdef}
\mathcal{L}_N(\mathcal{D})(\x)=\sum_{j=1}^M a_j\Phi_N(\x-\x_j).
\ee
We examine these constructions using a univariate example.
We consider $f(x)=|\cos x|$, and consider only the case when exact samples of $f$ are known at $128$ points. 
We note that for $x\in \TT$,
$$
f(x)=\frac{2}{\pi}\left\{1-\sum_{k=1}^\infty \frac{(-1)^k}{4k^2-1}\cos (2kx)\right\}.
$$
Therefore, the sequence $\{(k^2+1)^{s/2}\hat{f}(k)\}$ represents a sequence of Fourier coefficients of a continuous function for all $s<1$, but not if $s=1$ (and of course, not for any $s>1$). 
In particular, there is no ``optimal'' class with which to use the results of \cite{bdint} for this function.
\begin{uda}\label{dense_example}
{\rm
We consider $M=128$, $x_j=-\pi+2\pi j/M$, $j=0,\cdots, 127$, and denote the corresponding data by $\mathcal{D}_1$.
 We will refer to this choice as the \textbf{dense case}. 
 Figure~\ref{dense_fig} shows the errors at 1024 points on $[-\pi,\pi]$ for the operators $\mathcal{I}(\mathcal{D}_1)\in \HH_{64}^1$ and $\mathcal{L}_N(\mathcal{D}_1)\in \HH_{128}^1$ with $N=128$. It is noted that the collocation matrix for computing $\mathcal{I}(\mathcal{D}_1)$ is very ill-conditioned, but the matrix for computing $\mathcal{L}_N(\mathcal{D}_1)$ is well-conditioned. 
 Of course, in this very special example, one can compute $\mathcal{I}(\mathcal{D}_1)$ explicitly without having to solve a system of equations, but the intent of this example is to demonstrate the need to have methods more sophisticated than a straightforward interpolation.
\begin{figure}[ht]
\begin{center}
\begin{minipage}{0.3\textwidth}
\begin{center}
\includegraphics[width=\textwidth, height=0.75\textwidth]{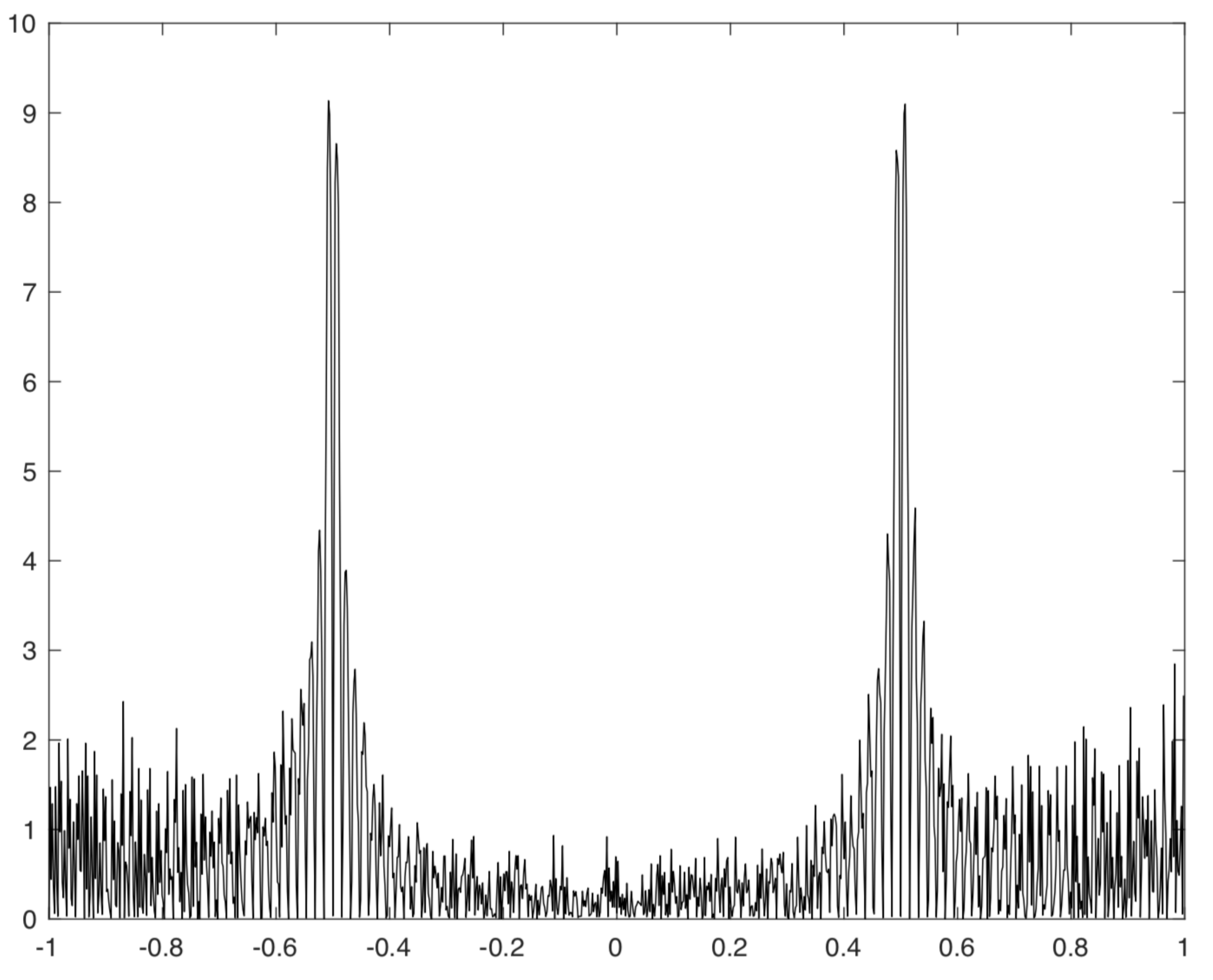} 
\end{center}
\end{minipage}
\begin{minipage}{0.3\textwidth}
\begin{center}
\includegraphics[width=\textwidth, height=0.75\textwidth]{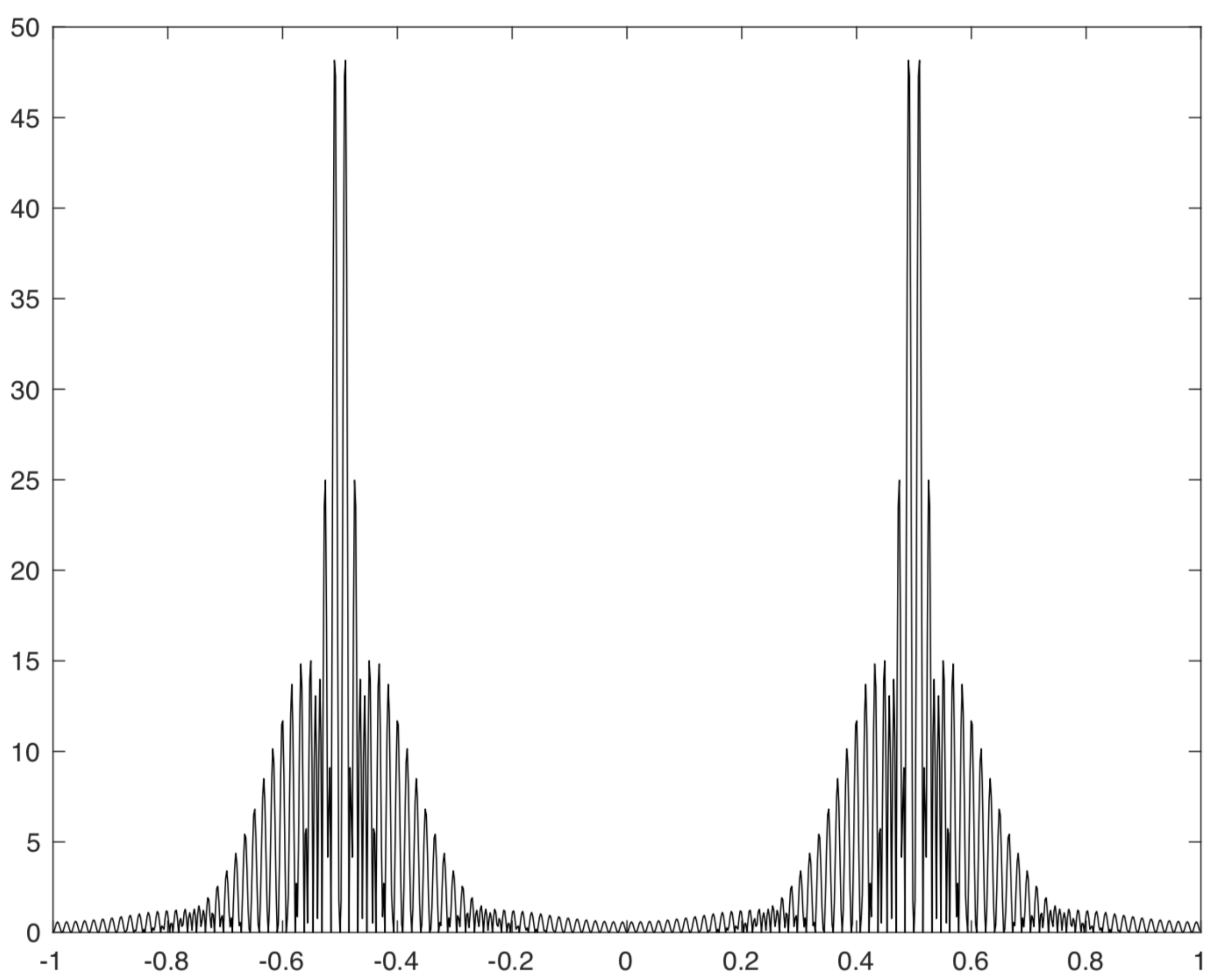} 
\end{center}
\end{minipage}
\end{center}
\caption{With $x_j=-\pi+2\pi j/M$, $j=0,\cdots, 127$, the differences at 1024 equidistant points of $[-\pi,\pi)$ between the true values $|\cos x|$ and the interpolatory polynomial of minimal degree, $\mathcal{I}(\mathcal{D}_1)(x)$ (left), and interpolation obtained by a higher degree localized kernel $\mathcal{L}_{128}(\mathcal{D}_1)(x)$ (right). The errors are magnified 1000 times and the points on the $x$-axis are multiples of $\pi$. The errors are similar near $\pm\pi/2$, but decrease rapidly to $0$ away from these in the figure on the right.}
\label{dense_fig}
\end{figure}
\qed}
\end{uda}
The situation is quite different when the data is not dense 
on $[-\pi,\pi)$.
\begin{uda}\label{not_dense_example}
{\rm
We consider $M=128$, $x_j=\pi/4+\pi j/(2M)$, $j=0,\cdots, 127$, and denote the corresponding data by $\mathcal{D}_2$.
 We will refer to this choice as the \textbf{non-dense case}. 
 Figure~\ref{non_dense_fig} shows the errors at 1024 points on $[-\pi,\pi]$ for the operators $\mathcal{I}(\mathcal{D}_2)$ and $\mathcal{L}_N(\mathcal{D}_2)$ with $N=128$, and $N=256$. As before, the collocation matrix for computing $\mathcal{I}(\mathcal{D}_2)$ is very ill-conditioned, but the matrix for computing $\mathcal{L}_N(\mathcal{D}_2)$ is well-conditioned. 
 However, a comparison between the middle and right figure of  Figure~\ref{non_dense_fig} shows the critical importance of choosing a degree higher than the minimal possible, commensurate with the minimal separation among the interpolation nodes.
 Also, we note in the right figure that the polynomial 
 $\mathcal{L}_{256}(\mathcal{D}_2)$ is highly localized, so that the errors on the interval $[\pi/4,\pi/2)$ are small, but those away from this interval are not; in fact, the polynomial is close to $0$ away from this interval.
\begin{figure}[ht]
\begin{center}
\begin{minipage}{0.3\textwidth}
\begin{center}
\includegraphics[width=\textwidth]{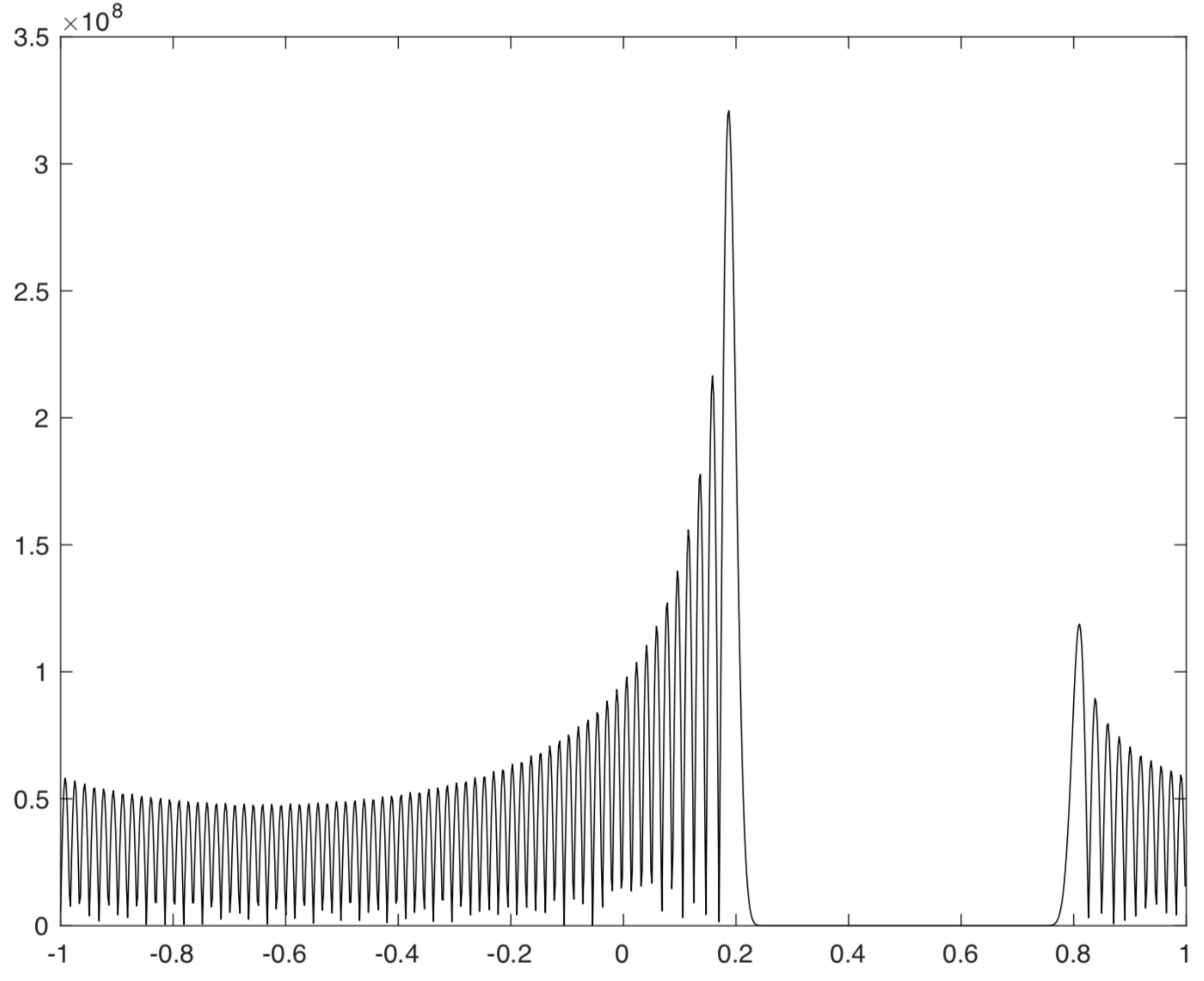} 
\end{center}
\end{minipage}
\begin{minipage}{0.3\textwidth}
\begin{center}
\includegraphics[width=\textwidth]{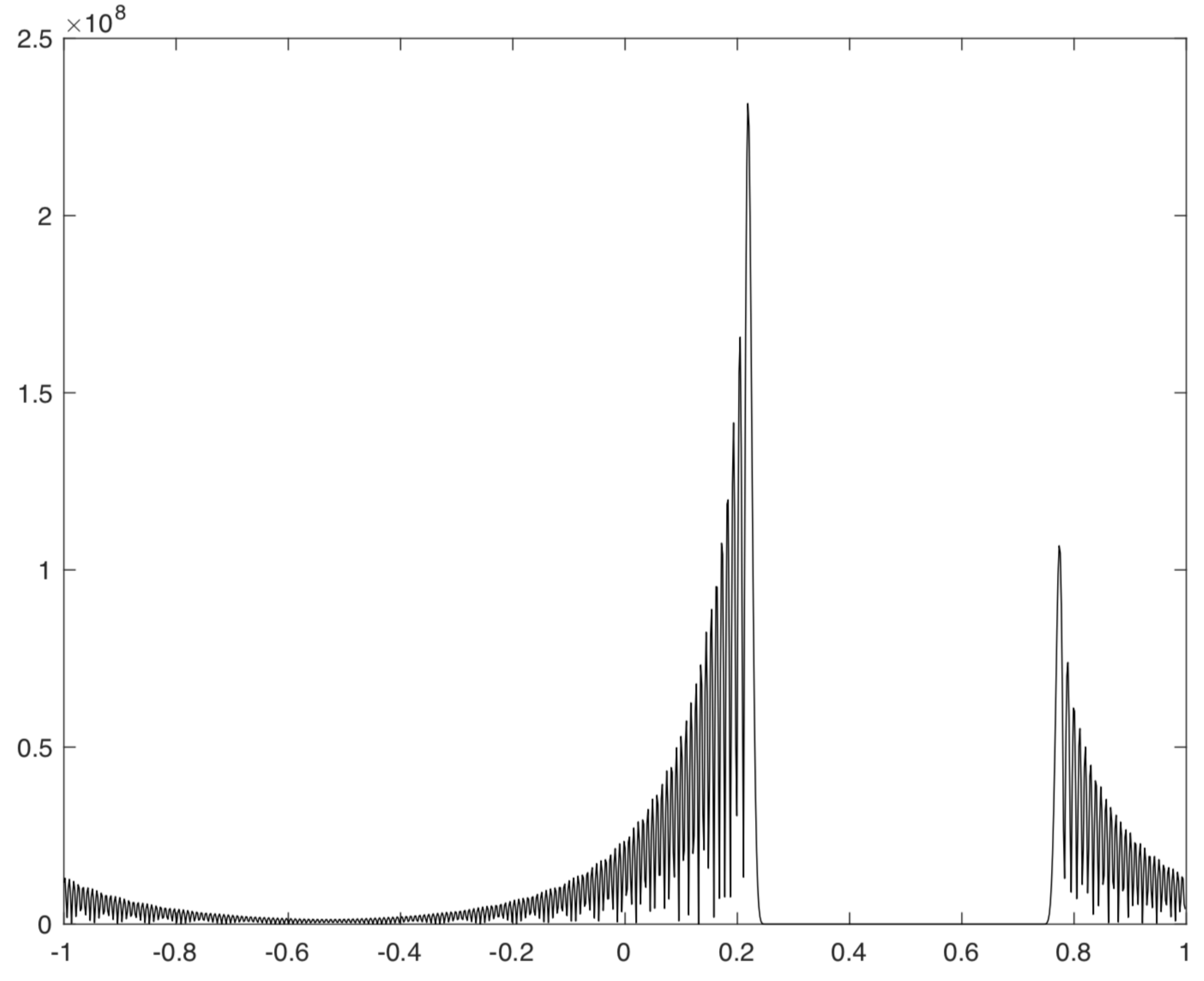} 
\end{center}
\end{minipage}
\begin{minipage}{0.3\textwidth}
\begin{center}
\includegraphics[width=\textwidth]{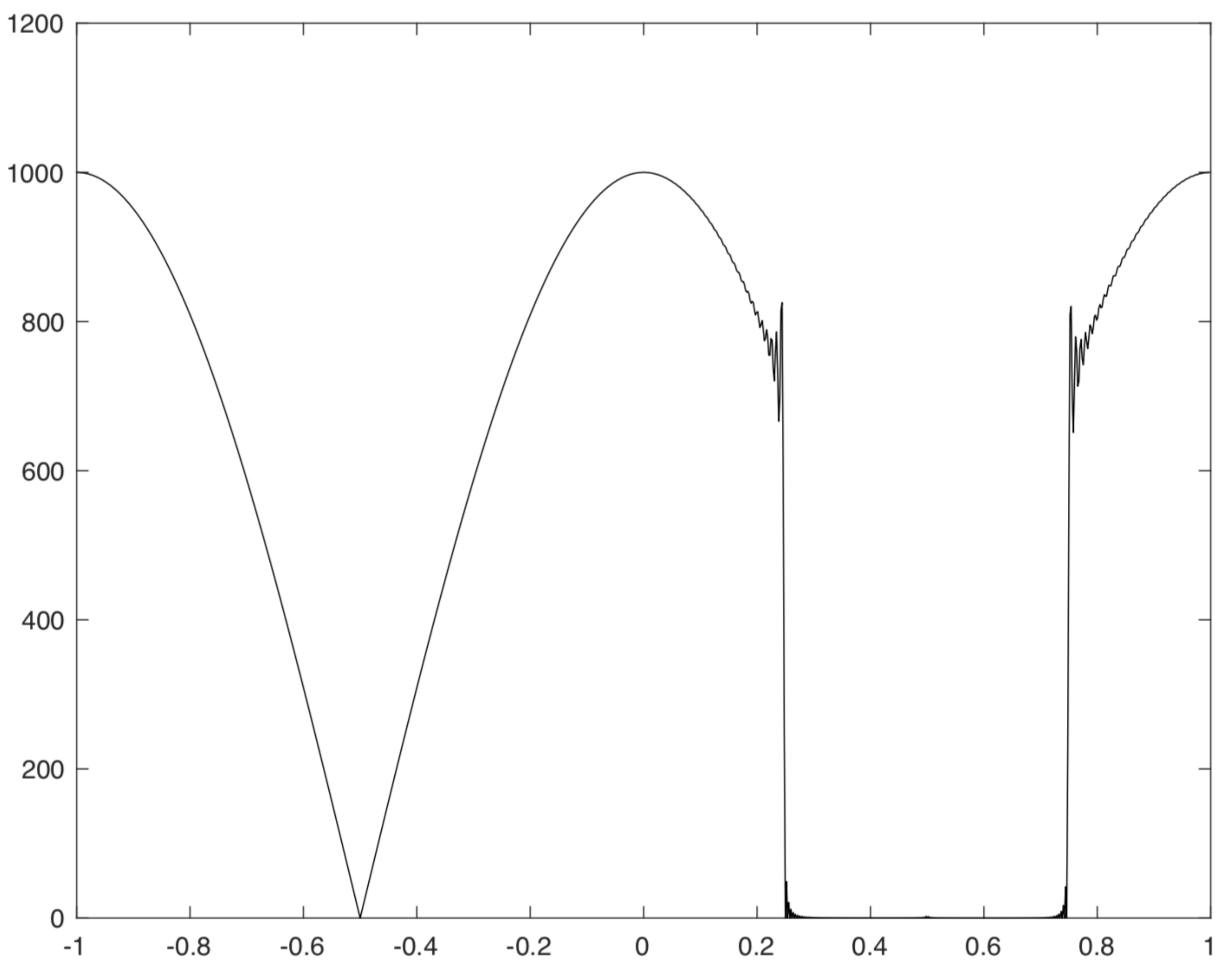} 
\end{center}
\end{minipage}
\end{center}
\caption{With $x_j=\pi/4+\pi j/(256)$, $j=0,\cdots, 127$, the differences at 1024 equidistant points of $[-\pi,\pi)$ between the true values $|\cos x|$ and the interpolatory polynomial of minimal degree, $\mathcal{I}(\mathcal{D}_2)(x)$ (left),  interpolation obtained by a higher degree localized kernel $\mathcal{L}_{128}(\mathcal{D}_2)(x)$ (middle), and $\mathcal{L}_{256}(\mathcal{D}_2)(x)$ (right). The errors are magnified 1000 times and the points on the $x$-axis are multiples of $\pi$. }
\label{non_dense_fig}
\end{figure}
\qed}
\end{uda}

As remarked in Remark~\ref{density_rmk}, it is not possible to achieve a bound analogous to \eref{szabad_bd} unless the training data is sufficiently dense on $\TT^q$ or unless more information about the target function is used in addition to its values at the training data. 
Therefore, let us now assume  that the Fourier coefficients $\hat{f}(\k)$ are known for all  $\k$ with $|\k|_2 <n$ for some $n>0$.
With this information, we can construct $\sigma_n(f)$ using \eref{summkerndef} with $n$ in place of $N$.
For a sufficiently large value of $N$, we may solve a system of equations
\be\label{interp_eqns}
\sum_{j=1}^M a_j\Phi_N(\x_\ell-\x_j) =y_\ell-\sigma_n(f)(\x_\ell), \qquad \ell=1,\cdots,M.
\ee
We then define
\be\label{blending_op_def}
\mathcal{T}_{n,N}^\#(\mathcal{D})(\x)=\sigma_n(f)(\x)+\sum_{j=1}^M a_j\Phi_N(\x-\x_j).
\ee

\begin{uda}\label{blending_example}
{\rm
We continue the examples with the target functions as in Examples~\ref{dense_example} and \ref{not_dense_example}. We estimate $\hat{f}(k)$ using a 128 point discrete Fourier transform. 
In the dense case, this is done using only the training data. 
In the non-dense case, the coefficients is the additional information we need, apart from the training data itself.
In the calculation of the operators $\mathcal{T}_{n,N}^\#$, we use $n=128$, $N=256$. The resulting errors are reported in Figure~\ref{blending_fig}.
We note that in the non-dense case, the errors are uniformly small on the entire interval $[-\pi,\pi)$, in contrast to the errors in the right-most figure in Figure~\ref{non_dense_fig}.
Also, the errors near the interpolation nodes are smaller than near the other singularity of the target function at $-\pi/2$.
In the non-dense case, we might as well use a larger number of Fourier coefficients to make $n=N$. We did this using 1024 point discrete Fourier transform, and $n=N=256$. The results are shown in Figure~\ref{blending_fig} as well.
\begin{figure}[ht]
\begin{center}
\begin{minipage}{0.3\textwidth}
\includegraphics[width=\textwidth]{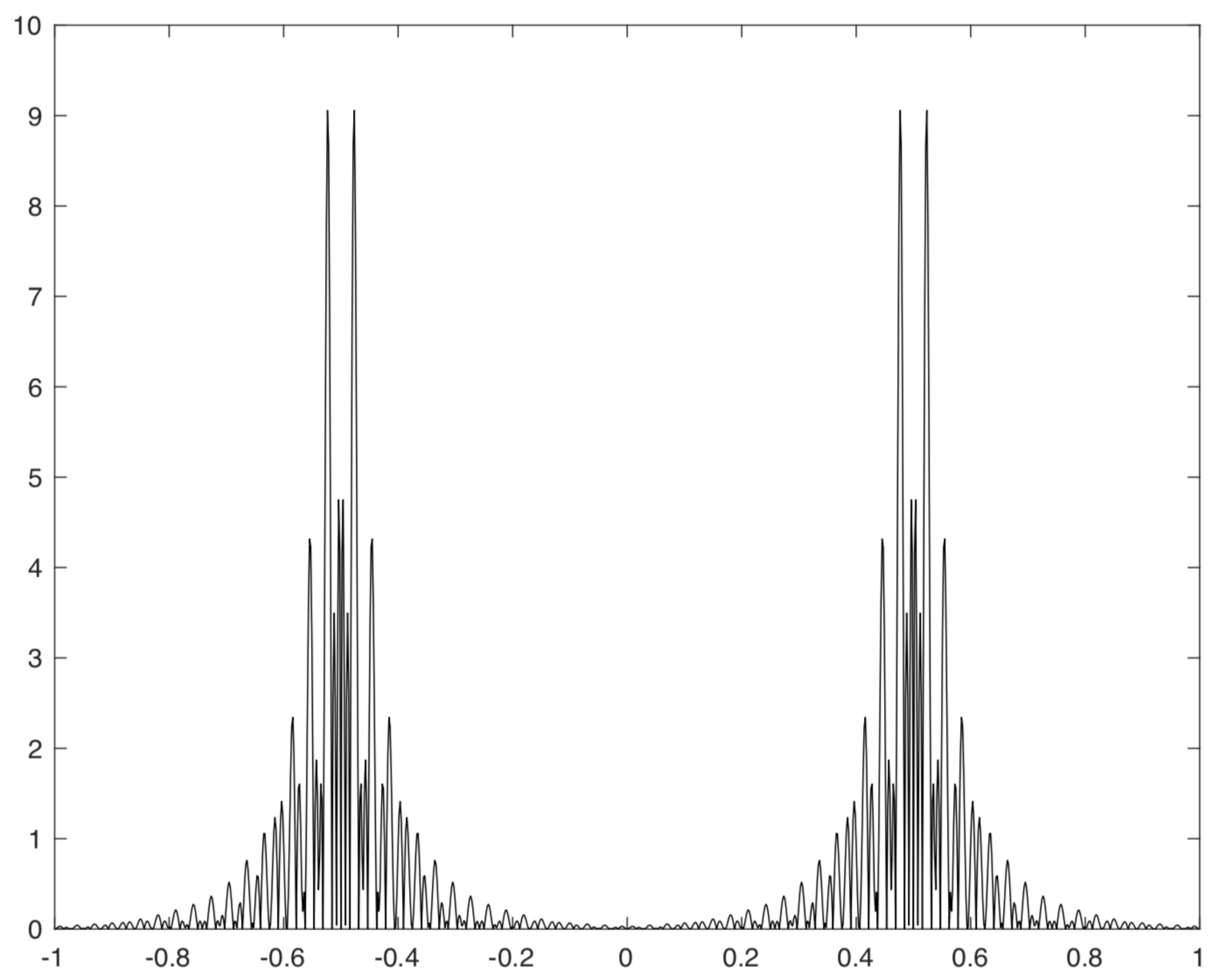} 
\end{minipage}
\begin{minipage}{0.3\textwidth}
\includegraphics[width=\textwidth]{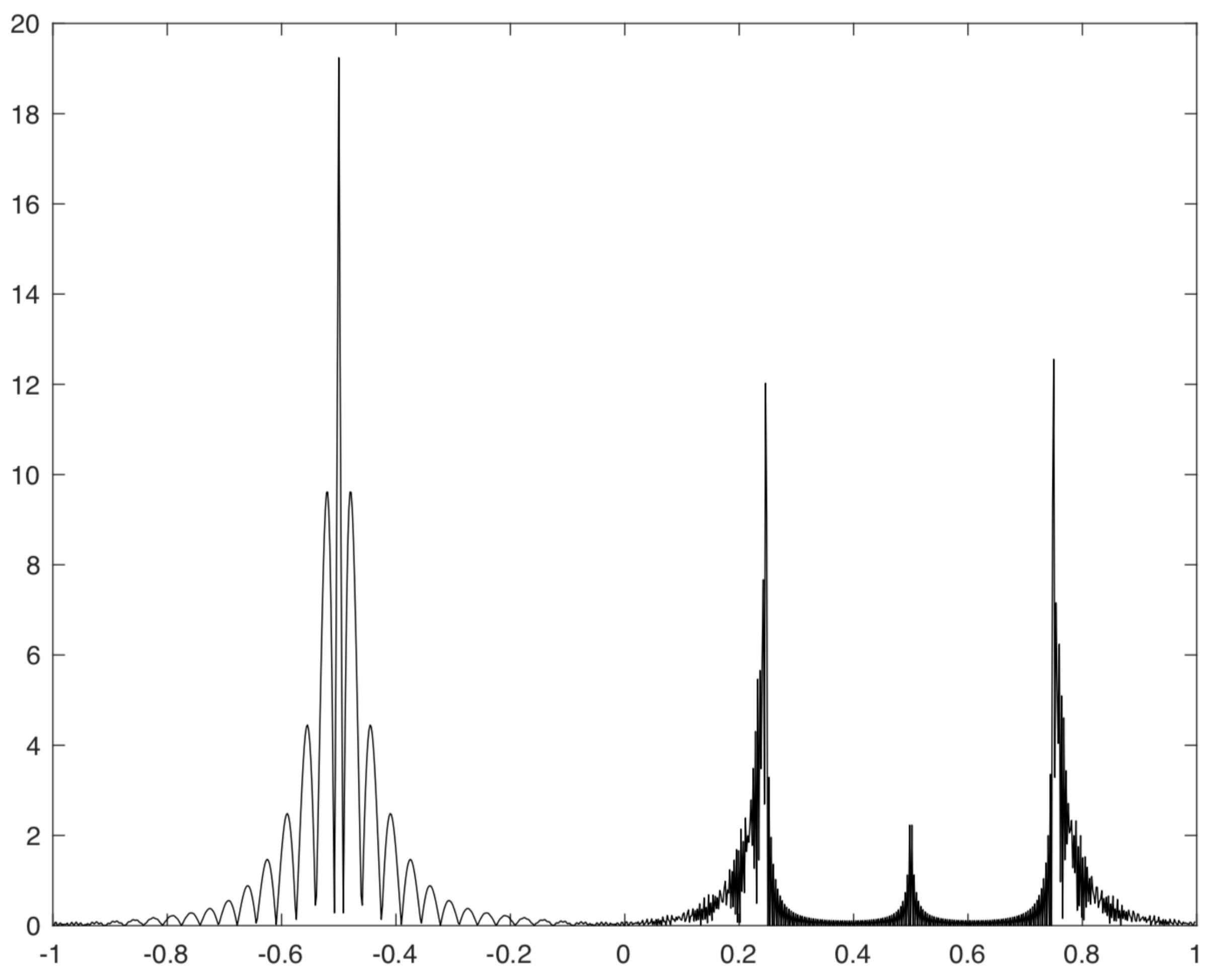} 
\end{minipage}
\begin{minipage}{0.3\textwidth}
\includegraphics[width=\textwidth]{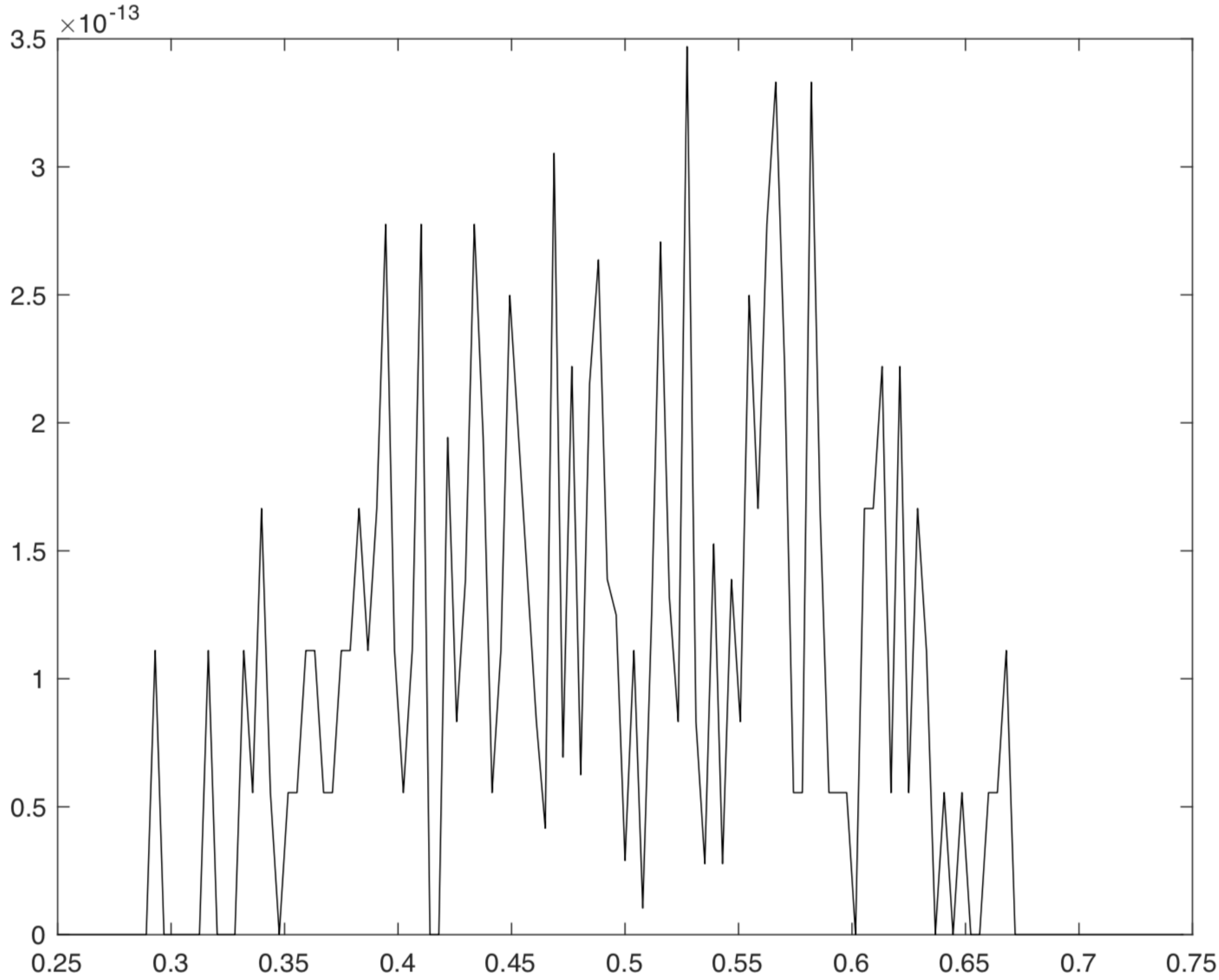} 
\end{minipage}
\begin{minipage}{0.3\textwidth}
\includegraphics[width=\textwidth]{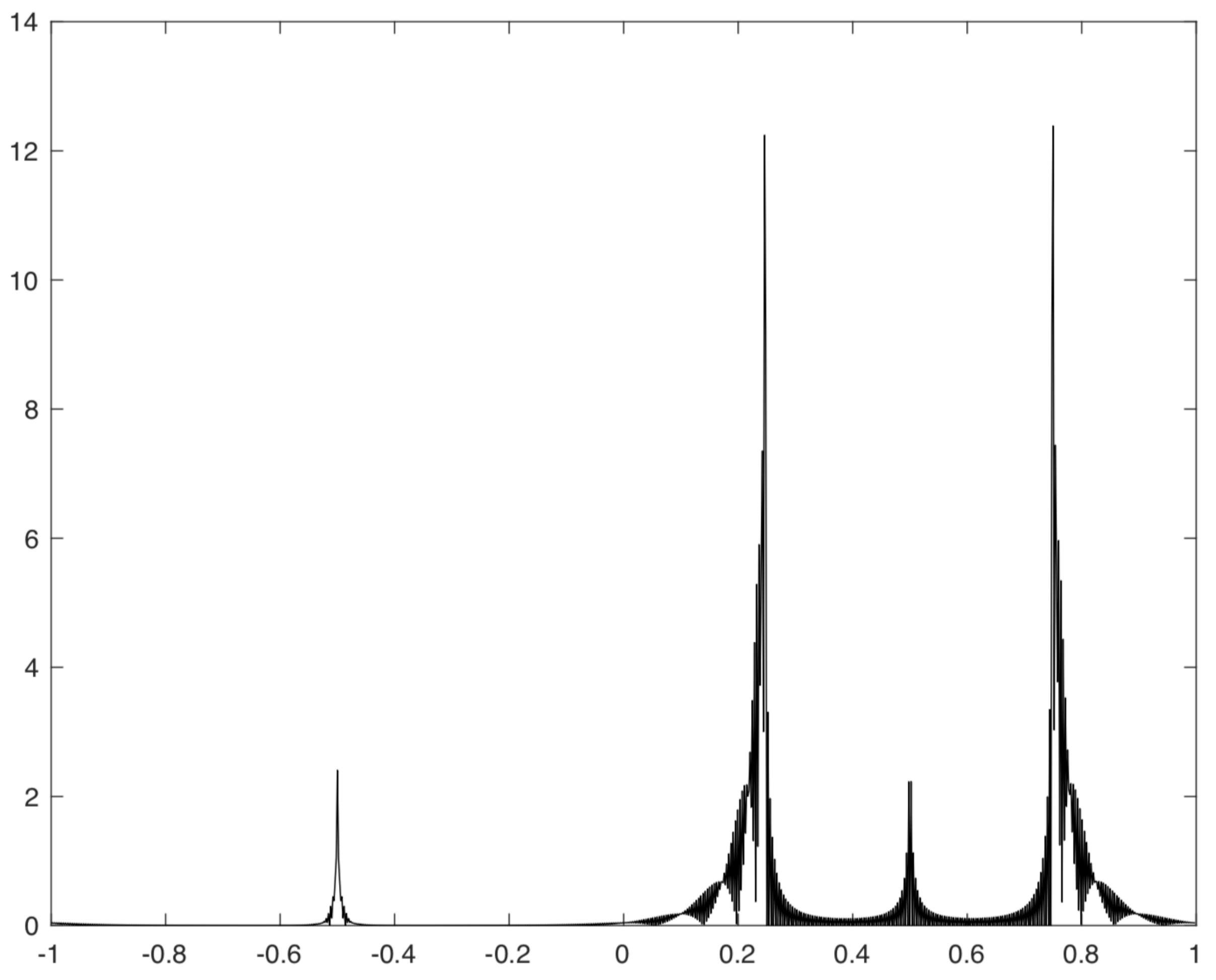} 
\end{minipage}
\begin{minipage}{0.3\textwidth}
\includegraphics[width=\textwidth]{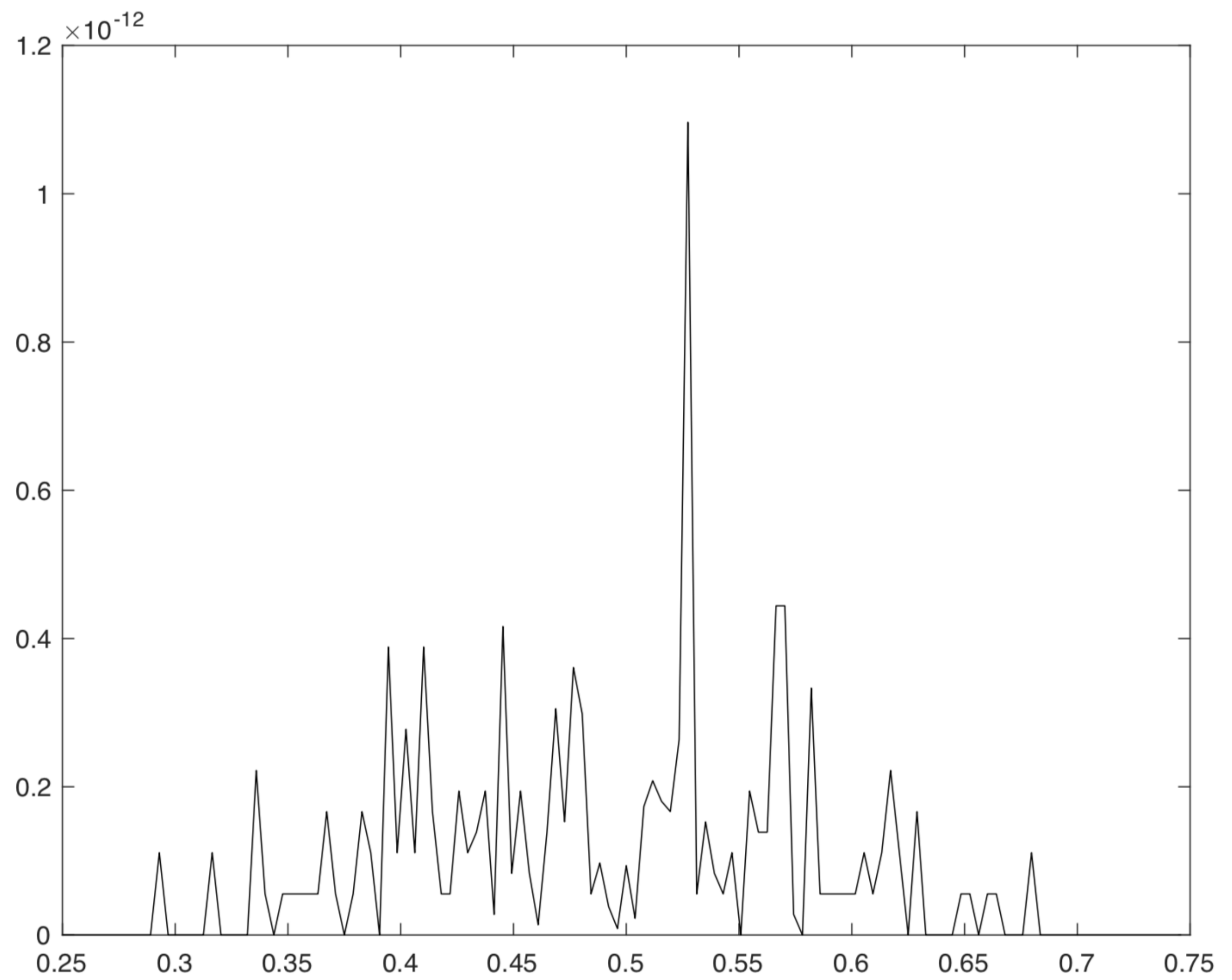} 
\end{minipage}
\end{center}
\caption{The errors at points of $[-\pi,\pi)$ using $\mathcal{T}_{128,256}^\#(\mathcal{D}_1)$ on the upper left, $\mathcal{T}_{128,256}^\#(\mathcal{D}_2)$ in the upper middle. The upper right figure is the error for $\mathcal{T}_{128,256}^\#(\mathcal{D}_2)$ at the interpolation nodes, to verify that the operator does the interpolation correctly. The bottom left and bottom right figures are analogous to the upper middle and upper right figures with $\mathcal{T}_{256,256}^\#(\mathcal{D}_2)$. All errors are maginfied 1000 times, and the $x$ axis has multiples of $\pi$.}
\label{blending_fig}
\end{figure}
\qed}
\end{uda}
 Theorem~\ref{feasibletheo} below is
 a generalization of Theorem~\ref{approxinttheo}, showing that the operators $\mathcal{T}_N^\#=\mathcal{T}_{N,N}^\#$ yield the requisite approximation.
Our proof is much simpler than that of Theorem~\ref{approxinttheo} in either \cite{approxint2002} or \cite{szabados1978some}.

\begin{theorem}\label{feasibletheo}
 There exists $B>0$ with the following property:  for  $f\in 
 C^*$ and $N\ge B\eta(\C)^{-1}$, the trigonometric polynomial $T^\#_N(\mathcal{D})=\mathcal{T}_{N,N}^\#(\mathcal{D})\in\HH_N^q$  in \eref{blending_op_def} is well defined, and satisfies
\be\label{interpbasic}
T^\#_N(\mathcal{D})(\x_j)= y_j, \qquad j=1,\cdots,M, \qquad \mbox{(Zero training error)}
\ee 
and
\be\label{degreebasic}
\|f-T^\#_N(\mathcal{D})\| \le c\left\{\epsilon+ E_{N/2}(f)\right\}. \qquad \mbox{(Good generalization error)}.
\ee
\end{theorem}

In view of the observations in Section~\ref{trigtonetsect}, the proof of the above theorem can be modified with neural network approximations at each stage to obtain the following version, where we overload the notation a bit for simplicity.

Suppose the following system of equations has a solution for some $\tilde{N}$ and $N$ (cf. \eref{trigneunetdef}):
\be\label{net_interp_eqns}
\sum_{j=1}^M a_j\mathbb{G}_{\tilde{N}}(\phi,\Phi_N)(\x_\ell-\x_j) =y_\ell-\mathbb{G}_{\tilde{N}}(\phi,\sigma_N(f))(\x_\ell), \qquad \ell=1,\cdots,M.
\ee
We then define
\be\label{net_blending_op_def}
\mathbb{G}_{\tilde{N},N}^\#(\mathcal{D})(\x)=\mathbb{G}_{\tilde{N}}(\phi,\sigma_N(f))(\x)+\sum_{j=1}^M a_j\mathbb{G}_{\tilde{N}}(\phi,\Phi_N)(\x-\x_j).
\ee
\begin{theorem}\label{net_feasibletheo}
There exist $B, \alpha^*>0$ with the following property. Let $f\in C^*$, $\phi\in C^*(\TT)$, $\hat{\phi}(1)\not=0$, $N\ge B\eta(\C)^{-1}$.  Let $\tilde{N}$ be such that
\be\label{netcond}
\|\Phi_N-\mathbb{G}_{\tilde{N}}(\phi,\Phi_N)\|\le \alpha^*.
\ee
Then the network $\mathbb{G}^\#_N(\mathcal{D})=\mathbb{G}_{\tilde{N},N}^\#(\mathcal{D})$  in \eref{net_blending_op_def} is well defined and satisfies
\be\label{net_interpbasic}
\mathbb{G}^\#_N(\mathcal{D})(\x_j)= y_j, \qquad j=1,\cdots,M, \qquad \mbox{(Zero training error)}
\ee 
and
\be\label{net_degreebasic}
\|f-\mathbb{G}^\#_N(\mathcal{D})\| \le c\left\{\epsilon+ E_{N/2}(f)+N^{q/2}E_{\tilde{N}}(1;\phi)\|f\|\right\}. \qquad \mbox{(Good generalization error)}.
\ee
\end{theorem}

\begin{rem}\label{shallow_number_rmk}
{\rm A volume comparison argument shows that the number $M$ of data points  satisfies $M\le c\eta(\C)^{-q}$. Thus, 
Theorems~\ref{feasibletheo} and ~\ref{net_feasibletheo} show that for a right configuration of the training data, a good generalization error as well as zero training error can be achieved by choosing the number of parameters proportional to the number of data points. It is demonstrated in \cite{belkin2018understand} that for RBF approximation, this phenomenon seems to hold in many applications with the number of parameters exactly equal to the number of data points.
}
\end{rem}

\begin{rem}\label{noise_rmk}
{\rm In practice, the training data is high dimensional and sparse; i.e., $\eta(\C)$ is large. 
The requirement that $N\ge B\eta(\C)^{-1}$ is therefore satisfied with moderate degrees $N$. 
}
\end{rem}

As demonstrated before, one cannot construct $T^\#_N(\mathcal{D})$ based only on the training data $\C$, unless $\C$ is sufficiently dense on $\TT^q$. 
We have already discussed an effort in the form of the operators $\mathcal{L}_N(\mathcal{D})$ in Example~\ref{not_dense_example}. 
An even simpler approach of just considering
$$
\frac{1}{\Phi_N(0)}\sum_{k=1}^M y_k\Phi_N(\circ-\x_k)
$$
yields similar bounds on the generalization error as those obtained by $\mathcal{L}_N(\mathcal{D})$. 
Of course, the training error for this simple construction is not $0$, but the localization properties of $\Phi_N$ ensure that it is small.
In general, if we anticipate a scenario where the training data sets become increasingly dense on some compact subset $K\subset \TT^q$, then we cannot expect convergence of  trigonometric polynomials that interpolate a noisy data, where the noise level does not decrease as well.

Another approach, described in \cite{bdint}, is to minimize a high order Sobolev norm of the trigonometric polynomial subject to the interpolatory conditions. 
This approach has been used to great advantage for a numerical solution of some notoriously hard partial differential equations in 2 or 3 dimensions.
However, the calculations are very ill-conditioned and require  very carefully designed algorithms.

We describe a softer regularization scheme that does not require high order Sobolev norms, and yields both good training and generalization errors. The generalization error is given point-wise, and is commensurate with the estimates given in Theorem~\ref{feasibletheo} when there is no noise.

 The space $W^*=W^*(\TT^q)$ consists of all continuously differentiable functions $f\in C^*$. We define
\be\label{sobolnormdef}
\|f\|_{W^*}=\|f\|_{W^*(\TT^q)} = \sum_{j=1}^q\|D_j f\|.
\ee
 For $n>0$ and $T\in \HH_n^q$, let
\be\label{regularizationdef}
R_n(T)=\max_{1\le j\le M} |y_j-T(\x_j)| + \frac{1}{n}\|T\|_{W^*}.
\ee

\begin{theorem}\label{constrtheo}
Let  $f\in W^*$,  $B$ be as in Theorem~\ref{feasibletheo},  $N\ge B\eta(\C)^{-1}$, and $T^*(\mathcal{D})=\disp\argmin_{T\in \HH_{N}^q}R_{N}(T)$. Then
\be\label{regerrest}
\max_{1\le j\le M}|y_j-\T^*(\mathcal{D})(\x_j)|\le \min_{T\in \HH_N^q}R_N(T) \le c\left\{\epsilon+\frac{
1}{N}\|f\|_{W^*}\right\}. \qquad \mbox{(Good training error)}
\ee
Let  $\x\in \TT^q$, and $\delta=\disp\min_{1\le j\le M}|\x-\x_j|$.  Then
\be\label{generalerr}
|f(\x)-T^*(\mathcal{D})(\x)| \le c(1+N\delta)\left\{\epsilon+\frac{
1}{N}\|f\|_{W^*}\right\}. \qquad \mbox{(Good generalization error)}.
\ee
\end{theorem}

\begin{rem}\label{usualMLrem}
{\rm
Theorem~\ref{constrtheo} makes no assumption on the target function except for differentiability. This is in contrast to usual machine learning theory, where one has to assume that the target function belongs to a reproducing kernel Hilbert space, with the kernel prescribed by the learning algorithm. 
}
\end{rem}
\begin{rem}\label{constrem}
{\rm
The estimate \eref{generalerr} shows that if $\x$ is very close to the training data so that $N\delta<1$, then the generalization error at $\x$ satisfies the same upper bound \eref{regerrest} that holds for $R_N(T^*)$. 
As remarked earlier in Remark~\ref{noise_rmk}, we have greater liberty in choosing  $N$ when the data is sparse. To take advantage of this fact, let $\x$ be such that $N\delta\ge 1$. Then \eref{generalerr} can be reformulated in the form
\be\label{generalerr_reform}
|f(\x)-T^*(\mathcal{D})(\x)|\le c\delta\left\{N\epsilon+\|f\|_{W^*}\right\}.
\ee
The term $\delta\|f\|_{W^*}$ is clearly a customary bound from numerical analysis in view of the mean value theorem. If $\epsilon \le \eta(\C)\|f\|_{W^*}$, it is possible to choose $N\sim \eta(\C)^{-1}$ so that error bound is $c\delta\|f\|_{W^*}$. 
}
\end{rem}
\begin{rem}\label{trigsupnorm}
{\rm
It is well known (cf. \cite[Chapter~X, Theorem~7.28]{zygmund} for the univariate case) that for any $T\in\HH_N^q$,
\be\label{discretization}
\|T\|\sim \max_{|\m|_\infty\le 3N-1}\left|T\left(\frac{2\pi \m}{3N}\right)\right|.
\ee
Therefore, 
$$
R_N(T)\sim\max_{1\le j\le M}|y_\ell-\sum_\k \hat{T}(\k)\exp(i\k\cdot\x_j)| +\frac{1}{N}\sum_{j=1}^q\max_{|\m|_\infty\le 3N-1}\left|\sum_\k k_j\hat{T}(\k)\exp\left(2\pi i\frac{\k\cdot\m}{3N}\right)\right|.
$$
If we replace $R_N(T)$ by the expression on the right hand side of the above equation, we get an optimization problem directly in terms of the coefficients $\hat{T}(\k)$. The theoretical results are not affected except for the actual values of the constants involved.
}
\end{rem}

\begin{rem}\label{directneunetrem}
{\rm
In this remark, let $T^*=\T^*(\mathcal{D})$.
 In view of \eref{trigneunetapprox}, the estimates \eref{regerrest} and \eref{generalerr} imply for any activation function $\phi$,
$$
\max_{1\le j\le M}|y_j-\mathbb{G}_{\tilde{N}}(\phi, T^*)(\x_j)|\le 
c\left\{\epsilon+\frac{
1}{N}\|f\|_{W^*}\right\}+\frac{4E_{\tilde{N}}(1;\phi)}{\hat{\phi(1)}}\sum_\k |\widehat{T^*}(\k)|,
$$
and
$$
|f(\x)-\mathbb{G}_{\tilde{N}}(\phi, T^*)(\x)| \le c(1+N\delta)\left\{\epsilon+\frac{
1}{N}\|f\|_{W^*}\right\}+\frac{4E_{\tilde{N}}(1;\phi)}{\hat{\phi(1)}}\sum_\k |\widehat{T^*}(\k)|
$$
respectively.
In particular, for the activation function obtained from the smooth ReLU function, the extra error terms decrease exponentially rapidly with $\tilde{N}$.
 It is therefore tempting to set up the regularization functional \eref{regularizationdef} directly with neural networks with free coefficients, weights, and thresholds to be determined by a suitable optimization technique. 
 However, the estimate \eref{trigneunetapprox} implies a strong connection between  these parameters for the network approximating a trigonometric polynomial. Therefore, the solution to such a direct approach with neural networks is not guaranteed to give the right training and testing errors.
}
\end{rem}

\bhag{Deep networks}\label{deepsect}

The following discussion regarding the terminology for  deep networks, including Figure~\ref{graphpict}, is based on the discussion in \cite{dingxuanpap}, and elaborates upon the same.

A commonly used definition of a deep network is the following. 
Let $\phi :\RR\to\RR$ be an activation function; applied to a vector $\x=(x_1,\cdots,x_q)$, $\phi(\x)=(\phi(x_1),\cdots,\phi(x_q))$. Let $L\ge 1$ be an integer, for $\ell=0,\cdots,L$, let $q_\ell\ge 1$  be an integer ($q_0=q$),  $T_\ell :\RR^{q_\ell}\to \RR^{q_{\ell+1}}$ be an affine transform, where $q_{L+1}=1$.  A deep network with $L$ hidden layers is defined as the compositional function
\be\label{usual_deep}
\x\mapsto T_L(\phi(T_{L-1}(\phi(T_{L-2}\cdots\phi(T_0(\x))\cdots).
\ee
This definition has several shortcomings. 
First, a function may have more than one compositional representation, as we will demonstrate shortly, so that the affine transforms and $L$ are not determined uniquely by the function itself. 
Second, this notion does not capture the connection between the nature of the target function and its approximation.
Third, the affine transforms $T_\ell$ define a special directed acyclic graph (DAG). 
It is cumbersome to describe notions of weight sharing, convolutions, sparsity, skipping of layers, etc. in terms of these transforms. 
Therefore, we have proposed in \cite{dingxuanpap}, to describe a deep network more generally as a directed acyclic graph (DAG) architecture.

Let $\mathcal{G}$ be a DAG, with the set of nodes $V\cup \mathbf{S}$, where $\mathbf{S}$ is the set of source nodes, and $V$ that of non-source nodes. A $\mathcal{G}$-function is defined as follows. Each of the  in-edges to each node in $V\cup \mathbf{S}$ represents an input real variable. 
For each node $v\in V\cup \mathbf{S}$, we denote its in-degree by $d(v)$. A node $v\in V\cup \mathbf{S}$ itself represents the evaluation of a real valued function $f_v$ of the $d(v)$ inputs. The out-edges fan out the result of this evaluation. 
Each of the source nodes obtains an input from some Euclidean space. 
Other nodes can also obtain such an input, but by introducing dummy nodes, it is convenient to assume that only the source nodes obtain an input from the Euclidean space. 

The notion of the level (or height) of a node is defined as follows.
The level of a source node is $0$; it represents a shallow network. 
The level of $v\in V$ is the length of the longest path from the nodes in $\mathbf{S}$ to $v$. 
The level of $v$ will be denoted by $H(v)$.

In \cite{dingxuanpap}, we have argued that deep networks display better approximation properties than shallow networks because they can take advantage of a compositional structure in the target function which shallow networks cannot.
However, compositionality is the property of an expression for the function; not an intrinsic property of a function itself. 
A simple example in the univariate case is the constant function $f(x)\equiv 2$, $x\in [0,1]$, that can also be expressed as a compositional function
$$f(x)=(x+1)\cosh\left(\log\left(\frac{2+\sqrt{3-2x-x^2}}{x+1}\right)\right), \qquad x\in [0,1].$$
It is not clear whether two different DAG structures can give rise to the same function. Even if we assume a certain DAG, it
is not clear that the choice of the constituent functions is uniquely determined for a given function on $\RR^q$. 
For example, one can write
$$
\cos^4 x= ((\cos x)^2)^2=(\cos^2 x)^2=(1/4)(1+\cos(2x))^2.
$$
The second expression above has the structure $h_1(h_2(h_3(x)))$, and the other two have the structures $g_1(g_2(x))$ or $f_1(f_2(x))$, both representing 
the same DAG but with different constituent functions. Thus, the question of whether a given multivariate function is in fact compositional cannot be answered.
Of course, for a given DAG, it is possible to use inverse/implicit function theorem  (in theory) in some cases to decide whether a family of functions are compositional according to the given DAG.

For our mathematical analysis, we therefore find it convenient to think of a $\mathcal{G}$-function as a set of functions $f=\{f_v :\RR^{d(v)}\to\RR\}_{v\in V\cup\mathbf{S}}$, rather than a single 
function on $\RR^q$. 
For example, the DAG $\mathcal{G}$ in Figure~\ref{graphpict} (\cite{dingxuanpap}) represents the compositional function
\bea\label{gfuncexample}
f^*(x_1,\cdots, x_9)&=&
h_{19}(h_{17}(h_{13}(h_{10}(x_1,x_2,x_3), h_{11}(x_4,x_5)), \nonumber\\ 
&& \qquad\qquad h_{14}(h_{10},h_{11}), h_{16}(h_12(x6,x_7,x_8,x_9)), h_{18}(h_{15}(h_{11},h_{12}),h_{16})).
\eea
The $\mathcal{G}$-function is $\{h_{10},\cdots,h_{19}= f^*\}$.
The individual functions $f_v$ will be called \textit{constituent functions}. 

We assume that there is only one sink node, $v^*$ (or $v^*(\mathcal{G})$) whose output is denoted by $f^*$ (the \emph{target function}). 
Technically, there are two functions involved here: one is the final output as a function of all the inputs to all source nodes, the other is the final output as a function of the inputs to the node $v^*$. 
We will use the symbol $f^*$ to denote both with comments on which meaning is intended when we feel that it may not be clear from the context. A similar convention is followed with respect to each of the constituent functions as well. 
For example, in the DAG of Figure~\ref{graphpict},
the function $h_{14}$ can be thought of both as a function of two variables, namely the outputs of $h_{10}$ and $h_{11}$ as well as a function of five variables $x_1,\cdots, x_5$.
In particular, if each constituent function is a neural network, $h_{14}$ is a shallow network receiving two inputs.

\begin{figure}[h]
\begin{center}
\includegraphics[scale=0.5]{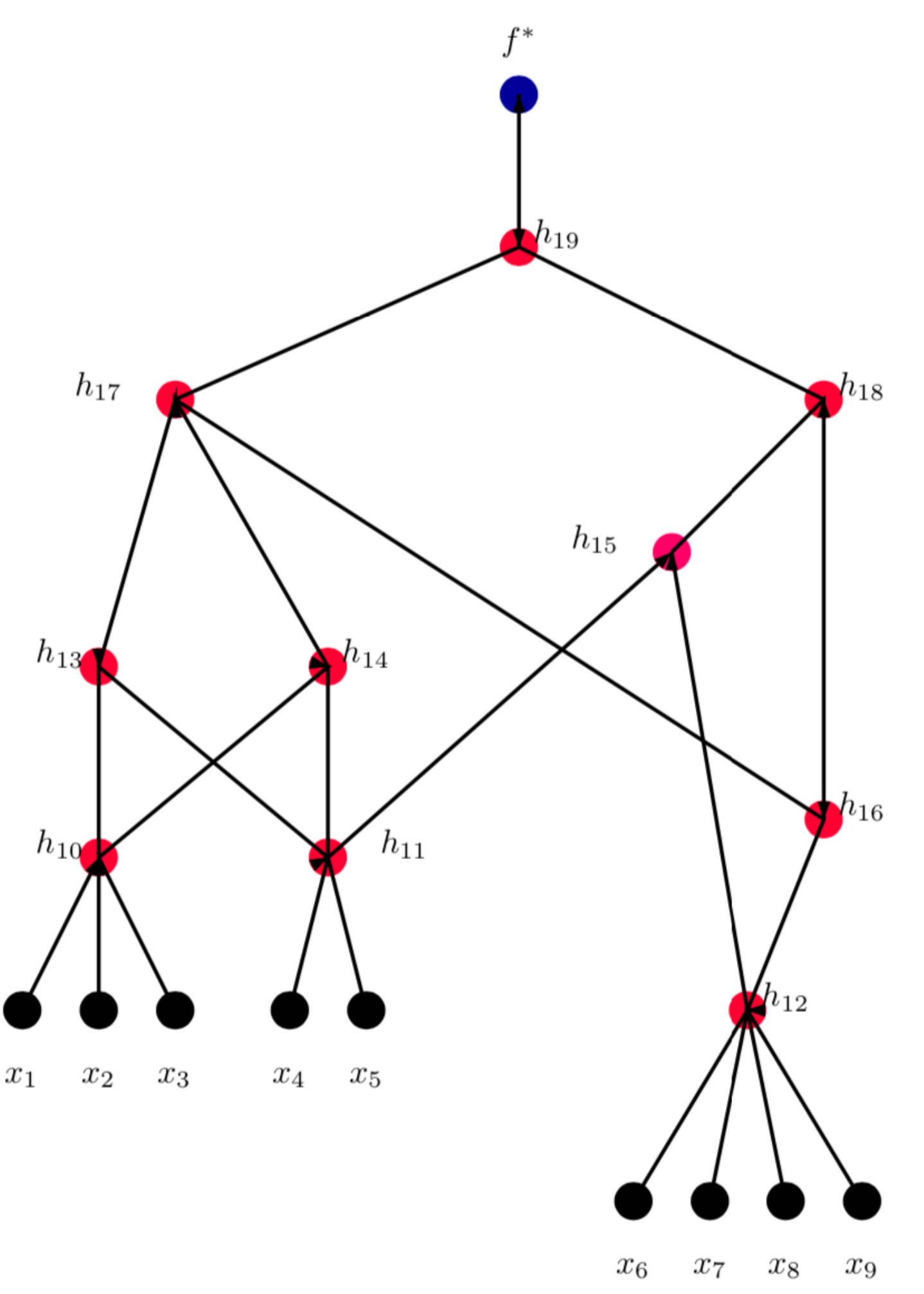} 

\end{center}

\caption{An example of a $\mathcal{G}$-function ($f^*$ given in \eref{gfuncexample}). The vertices of the DAG $\mathcal{G}$ are denoted by red dots. The black dots represent the input to the various nodes as indicated by the in--edges of the red nodes, and the blue dot indicates the output value of the $\mathcal{G}$-function, $f^*$ in this example.}
\label{graphpict}
\end{figure}

In this paper, we are interested only in the case where each of the inputs to each of the source nodes is in $\TT$ rather than $\RR$. 
Although this is no longer true for the non-source nodes, it is possible to accomplish this in the case when each of the constituent functions is continuous, as follows. We observe first that there is a one-to-one correspondence between functions on $[-1,1]^d$ and functions on $\TT^d$ that are even in each variable, given by
$$
F^\circ(\theta_1,\cdots,\theta_d)=F(\cos\theta_1,\cdots,\cos\theta_d), \qquad (\theta_1,\cdots,\theta_d)\in \TT^d.
$$
Let $f_u$ be one of the constituent functions.
With a re-normalization, we may  assume that the range of $f_u$ is a subset of $[-1,1]$. If $u_1,\cdots,u_{d(v)}$ are children of $v$ (i.e., if there is an edge from each $u_1,\cdots, u_{d(v)}$ to $v$), then $f_v$ can be seen as a function on $[-1,1]^d$, from which one can construct an even function $f_v^\circ$ on $\TT^d$ as just explained : informally, think of $f_v$ as an even function of the points $(\pm \arccos(f_{u_1}),\cdots, \pm\arccos(f_{u_{d(v)}})) \in\TT^d$.
Rather than complicating our notations, we will therefore assume in this paper that the domain of each constituent function is a torus, and the range is a subset of $\TT$. 
In particular, we may assume that every constituent function $f_v \in C^*(\TT^{d(v)})$. 

 We adopt the convention that for any function class $\XX(\TT^d)$, the class $\mathcal{G}$-$\XX$ denotes 
the set of $\mathcal{G}$-functions $f=\{f_v\}_{v\in V}$, where each constituent function $f_v\in \XX(\TT^{d(v)})$. We define 
\be\label{gengfuncnormdef}
\|f\|_{\XX,\mathcal{G}}=\sum_{v\in V\cup\mathbf{S}}\|f_v\|_{\XX(\TT^{d(v)})}.
\ee
Thus, for example,  $\mathcal{G}$-$W^*$ is the set of all $\mathcal{G}$-functions $f=\{f_v\}_{v\in V}$, where each constituent function $f_v\in \XX(\TT^{d(v)})$, and we write
$$
\|f\|_{W^*, \mathcal{G}}=\sum_{v\in V\cup\mathbf{S}}\|f_v\|_{W^*(\TT^{d(v)})}.
$$
For a vector $\mathbf{N}=(N_v)$, the symbol $\mathcal{G}$-$\HH_\mathbf{N}$ denotes the set of $\mathcal{G}$-functions $\{T_v\}_{v\in V}$, where each $T_v\in \HH_{N_v}^{d(v)}$, and for a $\mathcal{G}$-function $f=\{f_v\}_{v\in V}$, 
\be\label{gaddegapproxdef}
E_{\mathbf{N},\mathcal{G}}(f)=\sum_{v\in V\cup\mathbf{S}}E_{N_v}(d(v);f_v).
\ee

To make precise the various inputs to the constituent functions, we introduce some conventions. 
Let $\mathbf{S}=\{v_1,\cdots,v_s\}$, $q=\sum_{j=1}^s d(v_j)$. The input  $\x$ can be viewed as a vector in $\RR^q$, but also as an element of $\disp\prod_{j=1}^s \RR^{d(v_j)}$; i.e.,    $\x=((\x)_{v_1},\cdots,(\x)_{v_s})$ so that $(\x)_{v_j}\in\RR^{d(v_j)}$. 
We note that in making this statement, we have tacitly introduced some possible dummy variables here: for example, the function with the compositional structure
$$
F(x_1,x_2,x_3)=f(f_1(x_1,x_2), f_2(x_2, x_3))
$$
is viewed as a function of $4$ variables
$$
\tilde{F}(x_1,\cdots,x_4)=f(f_1(x_1,x_2), f_2(x_4, x_3))
$$
although $F$ is the restriction of $\tilde{F}$ to the hyper-plane $x_2=x_4$. 
In our opinion, it is a very deep question to figure out what domains the functions in practice are defined on, and to some extent, manifold learning addresses this issue.
Much of the approximation theory literature is however limited to approximation on cubes, spheres, and other known domains, torus in the present paper.
From a purely mathematical point of view, this is justified by an appeal to Stein's extension theorem \cite[Chapter~VI]{stein2016singular}.
Therefore, making this tacit introduction of dummy variables allows us to simplify notation and results without any theoretical loss of generality.

 Correspondingly, we define the sets 
$$
\C_{v_j}=\{(\x)_{v_j} : \x\in \C\}, \qquad j=1,\cdots, s.
$$
Thus, $\C_{v_j}$ is the training data ``seen'' by the source node $v_j$. 
This notion is extended recursively to other nodes of $\mathcal{G}$.
Let $v$  not be a source node,  $u_1, \cdots, u_{d(v)}$ be the children of $v$, $\C_{u_1},\cdots,\C_{u_{d(v)}}$ be the training data seen by these nodes. Thus, for any $\x\in\C$, the components given by $(\x)_{u_j}$ are seen by $u_j$, and
$$
\C_v=\left\{(f_{u_1}((\x)_{u_1}),\cdots,f_{u_{d(v)}}((\x)_{u_{d(v)}})) : 
\x\in \C\right\}.
$$
For example, in the DAG of Figure~\ref{graphpict}, the children of $h_{14}$ are $h_{10}$ and $h_{11}$. For each $\x\in\RR^9$, $h_{10}$ sees the components $(x_1,x_2,x_3)$, while $h_{11}$ sees the components $(x_4, x_5)$. We have
\begin{eqnarray*}
\lefteqn{\C_{h_{10}}=\{((\x)_1,(\x)_2,(\x)_3): \x\in \C\},\quad \C_{h_{11}}=\{((\x)_4,(\x)_5) : \x\in \C\},}\\
\C_{h_{14}}&=&\{((h_{10}((\x)_1,(\x)_2,(\x)_3), h_{11}((\x)_4,(\x)_5)): \x\in \C\}
\end{eqnarray*}
%
%
%

In this section, we state our theorems only for networks with activation function $t\mapsto \cos t$; i.e., networks that evaluate trigonometric polynomials.
The transition to networks with other activation functions is obtained by approximating these by trigonometric polynomials using \eref{trigneunetapprox} as in Section~\ref{shallowsect}.
This only adds to an additional  complication in the notation without adding anything new conceptually.

  The analogue of Theorem~\ref{feasibletheo} is the following.
\begin{theorem}\label{deepfeasibletheo}
Let $\mathcal{G}$ be a DAG with  sink node $v^*$. 
There exists $C=C(\mathcal{G})>0$ with the following property:  Let $f=\{f_v\}_{v\in V\cup\mathbf{S}}$ be a $\mathcal{G}$-function such that each of the constituent functions $\{f_v\}_{v\in V}$ is Lipschitz continuous with Lipschitz constant $\le L$. If $N_v\ge C(\mathcal{G})\eta(\C_v)^{-1}$, for each $v\in V$, and $\mathbf{N}=(N_v)$, the $\mathcal{G}$-function $T^\#_\mathbf{N}(\mathcal{D})\in\mathcal{G}$-$\HH_\mathbf{N}$ defined in \eref{deep_tshar_def}  satisfies
\be\label{deepinterpbasic}
(T^\#_\mathbf{N}(\mathcal{D}))_{v^*}(\x_j)= y_j, \qquad j=1,\cdots,M, \qquad \mbox{(Zero training error)}
\ee 
and
\be\label{deepdegreebasic}
\|f-T^\#_\mathbf{N}(\mathcal{D})\|_{C^*,\mathcal{G}} \le c(\mathcal{G},L)\left\{\epsilon+ E_{\mathbf{N}/2,\mathcal{G}}(f)\right\}. \qquad \mbox{(Good generalization error)}.
\ee
\end{theorem}

We recall that the polynomials in Theorem~\ref{feasibletheo} and hence, in Theorem~\ref{deepfeasibletheo} are constructed using the Fourier coefficients of the various functions involved. 
One could use the polynomials $\mathcal{L}_N(\mathcal{D})$ defined in \eref{locintdef} instead to get a construction based only on the values of the various functions computed using the training data.
However, since we do not know the constituent functions, this construction is not as constructive as with shallow networks. 
Besides, as with shallow networks, such a construction does not yield a uniform error bound analogous to \eref{deepdegreebasic}.

In contrast, Theorem~\ref{constrtheo} can be extended in a purely constructive manner, so as to yield a good training error and to keep the gradient of the resulting approximation under control so that pointwise generalization error bounds can be obtained.
Even though one does not know the constituent functions, 
one can construct a DAG trigonometric polynomial knowing the DAG structure and the training data alone.

For $\mathbf{n}=(n_v)$ and $T\in \mathcal{G}$-$\HH_\mathbf{n}$, let
\be\label{deepregularizationdef}
R_{\mathbf{n},\mathcal{G}}(T)=\max_{1\le j\le M} |y_j-T_{v^*}(\x_j)| + \sum_{v\in V}\frac{1}{n_v}\|T_v\|_{W^*(\TT^{d(v)})}.
\ee
In this definition, it is understood that $T_{v^*}$ is thought of as a function of the $q$-dimensional vector $\x$, but is computed using the all the constituent functions in $T$ using DAG structure prescribed by $\mathcal{G}$.

\begin{theorem}\label{deepconstrtheo}
Let $\mathcal{G}$ be a DAG with  sink node $v^*$. 
Let $f=\{f_v\}_{v\in V\cup\mathbf{S}}\in \mathcal{G}$-$W^*$, and $\max_{v\in V\cup\mathbf{S}}\|f_v\|_{W^*(\TT^{(d(v)})}\le L$. There exists $C(\mathcal{G})>0$ with the following property:  If $N_v\ge C(\mathcal{G})\eta(\C_v)^{-1}$, for each $v\in V$, and $\mathbf{N}=(N_v)$,
\be\label{deepregerrest}
\min_{T\in \mbox{$\mathcal{G}$-$\HH_\mathbf{N}$}}R_{\mathbf{N},\mathcal{G}}(T) \le c(\mathcal{G}, L)\left\{\epsilon+\sum_{v\in V}\frac{1}{N_v}\|f_v\|_{W^*(\TT^{d(v)})}\right\}. \qquad \mbox{(Good training error)}
\ee
Let $\x\in \TT^q$,  $\delta=\disp\min_{1\le j\le M}|\x-\x_j|$, $N=\max_{v\in V}N_v$.  Then with $$T^*(\mathcal{D})=\disp\argmin_{\mbox{$\mathcal{G}$-$\HH_\mathbf{N}$}}R_{\mathbf{N},\mathcal{G}}(T),$$ we have
\be\label{deepgeneralerr}
|f_{v^*}(\x)-(T^*(\mathcal{D}))_{v^*}(\x)| \le c(\mathcal{G}, L)\left\{\epsilon+\sum_{v\in V}\frac{1}{N_v}\|f_v\|_{W^*(\TT^{d(v)})} \right\} + c(\mathcal{G},L)\left(1+N\epsilon\right)^{H(v^*)}\delta. \qquad \mbox{(Good generalization error)}.
\ee
\end{theorem}

\begin{rem}\label{deeprmk}
{\rm
The theorems in this section indicate that the superiority of deep learning over shallow learning comes from two factors. 
One is that the compositional structure of the target function allows us to study the problem in a cascade of low dimensional problems.
The other is that this structure might allow us to sparsify the training data as we move up the cascade; noting that the minimal separation is defined only among the distinct points in a set.
}
\end{rem}

\bhag{Proofs}\label{pfsect}
\subsection{Preliminary results on trigonometric approximation}\label{trigsect}
We recall first some essential properties of  the kernels and operators defined in \eref{kerndef} and \eref{summkerndef}. Since the proofs of these facts are scattered among several of our publications,  we will sketch the proof in some detail for the sake of completion.

\begin{prop}\label{kernel_loc_prop}
Let $S >q$ be an integer. 
For $N\ge 1$,
\be\label{lockern}
|\Phi_N(\x)|\le \frac{cN^q}{\max(1, (N|\x|)^S)},  \qquad \x\in\TT^q,
\ee
and
\be\label{kernlowbd}
|\Phi_N(\bs 0)| \ge cN^q.
\ee
\end{prop}

\begin{proof}
 Our proof summarizes that of \cite[Theorem~6.1]{bdint}.
We consider the function $H(\t)=h(|\t|_2)$, $\t\in\RR^q$. Since $0\le H(\t)\le 1$ for all $\t\in\RR^q$, and $H(\t)=1$ for $|\t|_2\le 1/2$, it is clear that
\be\label{pf5eqn3}
|\Phi_N(\x)|\le \Phi_N(\bs 0) \sim N^q.
\ee
In particular, this proves \eref{kernlowbd}.
Since $h$ is constant in a neighborhood of $\bs 0$, it is easy to see that $H$ is $S$ times continuously differentiable as well, so that its Fourier transform satisfies
\be\label{pf5eqn1}
|\hat{H}(\u)|\le \frac{c(H)}{|\u|^S}, \qquad \u\not=\bs 0.
\ee
Hence, the Poisson summation formula yields
\be\label{pf5eqn2}
\Phi_N(\x)=N^q\sum_{\j\in\ZZ^q}\hat{H}(N(\x+2\j\pi)), \qquad \x\in\TT^q.
\ee
For $\x\not=0$, $\j\not=\bs 0$, 
$$
 |\x+2\j\pi| \ge 2|\j|_\infty\pi-|\x|\ge (2|\j|_\infty-1)|\x|.
$$
Therefore, \eref{pf5eqn1} and \eref{pf5eqn2} show that for $\x\not=\bs 0$,
$$
|\Phi_N(\x)|\le N^q\left\{|\hat{H}(N\x)|+\sum_{\j\in\ZZ^q, \j\not=\bs 0}|\hat{H}(N(\x+2\j\pi))|\right\} \le c(H)N^q\left\{\frac{1}{(N|\x|)^S}+\sum_{\j\in\ZZ^q, \j\not=\bs 0}\frac{1}{((2|\j|_\infty-1)N|\x|)^S}\right\}.
$$
Since $S>q$, the infinite series converges. 
Therefore,
\be\label{lockernstrong}
|\Phi_N(\x)|\le c\frac{N^q}{(N|\x|)^S}, \qquad \x\not=\bs 0.
\ee
Together with \eref{pf5eqn3}, this  leads to \eref{lockern}.
\end{proof}

\begin{prop}\label{kernelopsummary_prop}
Let $S >q$ be an integer.\\
\textrm{(a)} 
    There exists a constant $B>0$ with the following property:  If $\C=\{\x_1,\cdots,\x_M\}$, and $N\ge B\eta(\C)^{-1}$ then  for $\x\in\TT^q$,
\be\label{phinsumest}
\sum_{j=1}^M |\Phi_N(\x-\x_j)| \le cN^q. 
\ee
and
\be\label{greshgorin}
\sum_{j: |\x_j-\x|\ge \eta(\C)} |\Phi_N(\x-\x_j)|\le (1/2)\Phi_N(\bs 0)=(1/2)\Phi_N(\x-\x)\le cN^q.
\ee

\textrm{(b)} If $T\in\HH_{N/2}^q$ then $\sigma_N(T)=T$. Further, 
\be\label{opbd}
\|\sigma_N(f)\| \le c\|f\|, \qquad f\in C^*.
\ee
Consequently,
\be\label{goodapprox}
E_N(f)\le \|f-\sigma_N(f)\|\le cE_{N/2}(f), \qquad f\in C^*.
\ee
\end{prop}

The two parts of the proposition are proved along the same lines. 
It is convenient to prove the following lemma (cf. \cite[Proposition~5.1]{eignet}), which we will use once with $\nu$ being the measure that associates the mass $1$ with each $\x_j$ for part (a), and  once with $\nu$ being the Lebesgue measure $\mu^*$ on $\TT^q$ for part (b).

\begin{lemma}\label{regularlemma}
Let $\nu$ be a positive measure on $\TT^q$, $d\ge 0$, and for some $A_\nu >0$,
\be\label{regularmeasdef}
\nu(\BB(\x,r))\le A_\nu(r+d)^q, \qquad r>0, \ \x\in\TT^q.
\ee 
Then for $r\ge 1/N$,
\be\label{nutailest}
\int_{\TT^q\setminus \BB(\x,r)} |\Phi_N(\x-\u)|d\nu(\u) \le cA_\nu(Nr)^{q-S}(1+d/r)^q,
\ee
where $c$ is a positive constant depending only on $q$, $S$, and $h$ but not on $r$, $d$, or $\nu$.
\end{lemma}
 \begin{proof}\ 
In this proof, we assume by re-normalization that $A_\nu=1$. Also, we write $\mathbb{A}_k=\BB(\x,2^{k+1}r)\setminus \BB(\x, 2^kr)$.
Then \eref{regularmeasdef} shows that 
$\nu(\mathbb{A}_k)\le (2^{k+1}r+d)^q\le 2^q 2^{kq}(r+d)^q$.  
Using \eref{lockernstrong}, we conclude that
$$
\int_{\TT^q\setminus \BB(\x,r)} |\Phi_N(\x-\u)|d\nu(\u)\le \sum_{k=0}^\infty \int_{\mathbb{A}_k} |\Phi_N(\x-\u)|d\nu(\u)\le c2^q(Nr)^{q-S}(1+d/r)^q\sum_{k=0}^\infty 2^{(q-S)k}.
$$
 \end{proof}

\noindent\textbf{Proof of Proposition~\ref{kernelopsummary_prop}.}\\
To prove part (a), let $\nu$ be a measure that associates the mass $1$ with each $\x_j$, $j=1,\cdots, M$, and let $\eta=\eta(\C)$.  We claim that $\nu$ satisfies \eref{regularmeasdef} with $d=\eta$ and $A_\nu=c\eta^{-q}$ for some $c>0$ depending only on $q$.
Fix $\x$ and $r$. If $r<\eta/2$, then $\BB(\x,r)\cap \C$ can contain at most $1$ point. Therefore the claim is satisfied trivially.
Let $r\ge \eta/2$, and $\BB(\x,r)\cap \C=\{\x_1,\cdots,\x_J\}$. 
We observe that $\mu^*(\BB(\y,s))=cs^q$ for all $\y\in\TT^q$ and $s\in (0,\pi)$. 
Since the balls $\BB(\x_j,\eta/3)$ are disjoint, we have
$$
\nu(\BB(\x,r))=|J|=c\eta^{-q}\sum_{j=1}^J \mu^*(\BB(\x_j,\eta/3) )=c\eta^{-q}\mu^*\left(\cup_{j=1}^J \BB(\x_j,\eta/3)\right) \le c\eta^{-q}\mu^*(\BB(\x, r+\eta/3)) \le c_1\eta^{-q}(r+\eta)^q.
 $$
This proves the claim.
Thus, we may use \eref{nutailest} with $r=d=\eta$, $A_\nu=c\eta^{-q}$ to obtain for $N\ge \eta^{-1}$ that
$$
\sum_{j: |\x-\x_j|\ge \eta}|\Phi_N(\x-\x_j)|=\int_{\TT^q\setminus \BB(\x,\eta)}|\Phi_N(\x-\u)|d\nu(\u)\le c\eta^{-q}(N\eta)^{q-S}=cN^q(N\eta)^{-S}\le c\Phi_N(\bs 0)(N\eta)^{-S}.
$$
We choose $B>0$ such that if $N\eta\ge B$, $c(N\eta)^{-S}\le 1/2$. 
This proves \eref{greshgorin}. 
Further, since any ball $\BB(\x,\eta)$ contains at most $2^q$ points of $\C$, \eref{greshgorin} and \eref{pf5eqn3} lead to \eref{phinsumest}.

To prove part (b), we use Lemma~\ref{regularlemma} with $\mu^*$ in place of $\nu$. 
Clearly, \eref{regularmeasdef} is satisfied with $d=0$, and $A_{\mu^*}=c$. 
Therefore, using \eref{pf5eqn3} and \eref{nutailest} (with $r=1/N$) leads to
$$
\int_{\TT^q}|\Phi_N(\x-\u)|d\mu^*(\u)=
\int_{\BB(\x,1/N)}|\Phi_N(\x-\u)|d\mu^*(\u)+ \int_{\TT^q\setminus \BB(\x,1/N)}|\Phi_N(\x-\u)|d\mu^*(\u)\le c.
$$
It is now easy to deduce \eref{opbd}. 
If $T\in\HH_{N/2}^q$, then $\hat{T}(\k)=0$ if $|\k|_2\ge N/2$, while $h(t)=1$ if $|t|\le 1/2$. 
It follows from the definition \eref{summkerndef} that   $\sigma_N(T)=T$. 
For any $f\in C^*$ and $T\in\HH_{N/2}^q$, \eref{opbd} leads to
$$
E_N(f)\le \|f-\sigma_N(f)\|=\|f-T-\sigma_N(f-T)\|\le c\|f-T\|.
$$
This implies \eref{goodapprox}. \qed

\begin{prop}\label{matrix_inv_prop}
Let $\C$, $B$, and $N$ be as in Proposition~\ref{kernelopsummary_prop}, $\Psi\in C^*$, and $\mathbf{B}$ be the matrix $[\Psi(\x_j-\x_k)]_{j,k=1}^M$. 
There exists $\alpha^*>0$  such that if
\be\label{phinapprox}
\|\Phi_N-\Psi\|\le \alpha^*,
\ee
then the matrix $\mathbf{B}$ is invertible, and $\|\mathbf{B}^{-1}\|_{\ell^\infty\to\ell^\infty}\le cN^{-q}$.
In particular, let
 $\mathbf{b}\in \RR^M$.  Then there exists (uniquely) $\mathbf{a}\in \RR^M$ such that 
\be\label{mateqn}
\sum_{j=1}^Ma_j\Psi(\x_\ell-\x_j)=b_\ell, \qquad \ell=1,\cdots,M,
\ee
and
\be\label{invmatnorm}
\max_{1\le j\le M}|a_j| \le cN^{-q}\max_{1\le \ell\le M}|b_\ell|.
\ee
\end{prop}

\begin{proof}\ 
When $\Psi=\Phi_N$, the proposition follows from \eref{kernlowbd}, \eref{greshgorin}, and standard facts from linear algebra, (cf. \cite[Proposition~6.1]{eignet}).
In this proof, the matrix norm $\|\cdot\|$ will refer to the norm $\|\cdot\|_{\ell^\infty\to\ell^\infty}$, and we write $\eta=\eta(\C)$. 
Denoting the matrix $[\Phi_N(\x_j-\x_k)]_{j,k=1}^M$ by $\mathbf{A}$, we have observed that $\|\mathbf{A}^{-1}\|\le cN^{-q}$.
Also, it is easy to see  that $M\le c\eta^{-q}$. Therefore, recalling that $N\eta\ge B$, 
\begin{eqnarray*}
\|\mathbf{A}-\mathbf{B}\|&=&\max_j\sum_{k=1}^M|\Phi_N(\x_j-\x_k)-\Psi(\x_j-\x_k)| \le M\|\Phi_N-\Psi\|\le c\eta^{-q}\|\Phi_N-\Psi\|\\
&\le& cB^{-q}N^q\|\Phi_N-\Psi\|\le \frac{cB^{-q}}{\|\mathbf{A}^{-1}\|}\|\Phi_N-\Psi\|.
\end{eqnarray*}
We now choose $\alpha^*$ so that \eref{phinapprox} implies
$$
\|\mathbf{A}-\mathbf{B}\|\le \frac{1}{2\|\mathbf{A}^{-1}\|}.
$$
A perturbation theorem from linear algebra \cite[Theorem~2.3.4]{golub2012matrix} then shows that
$$
\|\mathbf{A}^{-1}-\mathbf{B}^{-1}\|\le \|\mathbf{A}^{-1}\|\le cN^{-q}.
$$
This shows that $\|\mathbf{B}^{-1}\|\le cN^{-q}$.
\end{proof}

We also need some results about approximation of a function and its derivatives.
We recall a theorem from \cite[Theorem~$1^\circ$]{czipser1957approximation}.  Per Section~\ref{notationsect}, the notation $\|\circ\|_1$ indicates the univariate supremum norm.

\begin{theorem}\label{czipserfreudtheo}
Let $q=1$,  $f\in W^*(\TT)$, $n\ge 1$ be an integer, $E>0$, and $T\in\HH_n^1$ satisfy
\be\label{unitrigapprox}
\|f-T\|_1 \le E.
\ee
Then 
\be\label{unitrigderapprox}
\|f'-T'\|_1 \le c\left\{n E + E_{n}(f')\right\}.
\ee
\end{theorem}

In the multivariate case, we take the derivatives one variable at a time to deduce the following corollary of Theorem~\ref{czipserfreudtheo}.

\begin{cor}\label{czipserfreudcor}
Let  $f\in W^*$, $N\ge 1$ be an integer, $E>0$, and $T\in\HH_N^q$ satisfy
\be\label{trigapprox}
\|f-T\| \le E.
\ee
Then  
\be\label{trigderapprox}
\|f-T\|_{W^*} \le c\left\{N E + \sum_{j=1}^q E_N(D_j f)\right\}.
\ee
In particular,
\be\label{difftrigest}
\|T\|_{W^*} \le cN\left\{ E +\frac{1}{N}\|f\|_{W^*}\right\}.
\ee
\end{cor}

We note also the direct theorem of trigonometric approximation \cite[Section~5.3]{timanbk}.

\begin{prop}\label{directtheoprop}
If $f\in W^*$ then
\be\label{favardest}
E_N(f)\le \frac{c}{N}\|f\|_{W^*}.
\ee
\end{prop}

\subsection{Proof of Proposition~\ref{periodic_relu_prop}.}
Using \eref{trigneunetapprox} and Schwarz inequality, we obtain that
\be\label{pf4eqn1}
\begin{aligned}
\|\sigma_n(f)-\mathbb{G}_N(\phi,\sigma_n(f))\| &\le c(\phi)E_N(1;\phi)\sum_{\k\in\ZZ^q}\left|h\left(\frac{|\k|_2}{n}\right)\hat{f}(\k)\right|\\ &\le c(\phi)E_N(1;\phi)\left\{\sum_{\k\in\ZZ^q}\left(h\left(\frac{|\k|_2}{n}\right)\right)^2\right\}^{1/2}\left\{\sum_{\k\in\ZZ^q}
|\hat{f}(\k)|^2\right\}^{1/2}.
\end{aligned}
\ee
Since $h(t)=0$ if $t\ge 1$, and $0\le h(t)\le 1$ for all $t$, 
$$
\sum_{\k\in\ZZ^q}\left(h\left(\frac{|\k|_2}{n}\right)\right)^2\le cn^q.
$$
In view of Bessel inequality, \eref{pf4eqn1} now implies
$$
\|\sigma_n(f)-\mathbb{G}_N(\phi,\sigma_n(f))\|\le c(\phi)E_N(1;\phi)n^{q/2}\left(\int_{\TT^q}|f(\x)|^2d\x\right)^{1/2}\le c(\phi)E_N(1;\phi)n^{q/2}\|f\|.
$$
Together with \eref{goodapprox}, this leads to 
\eref{relu_approx}. \qed
\subsection{Proof of the theorems in Section~\ref{shallowsect}.}\label{pf_feasible_sect}

\noindent\textbf{Proof of Theorem~\ref{feasibletheo}.} \\

In this proof, let $B$ be as in Proposition~\ref{kernelopsummary_prop}(a). $N\ge B\eta(\C)^{-1}$,
and for $\ell=1,\cdots, M$,
$z_\ell=y_\ell-\sigma_N(f)(\x_\ell)$.  Proposition~\ref{matrix_inv_prop} then guarantees that there exist $a_j\in\RR$ such that 
\be\label{pf1eqn1}
\sum_{j=1}^Ma_j\Phi_N(\x_\ell-\x_j) = z_\ell=y_\ell-\sigma_N(f)(\x_\ell), \qquad \ell=1,\cdots,M,
\ee
and (cf. \eref{invmatnorm} and \eref{goodapprox})
\bea\label{pf1eqn2}
\max_{1\le j\le M} |a_j| &\le& cN^{-q}\max_{1\le\ell\le M}|z_\ell|=cN^{-q}\max_{1\le \ell\le M}|y_\ell-\sigma_N(f)(\x_\ell)|\nonumber\\
&\le& cN^{-q}\left\{\max_{1\le \ell\le M}|y_\ell-f(\x_\ell)| +\max_{1\le \ell\le M}|f(\x_\ell)-\sigma_N(f)(\x_\ell)|\right\}\nonumber\\
&\le& cN^{-q}\left\{\epsilon+E_{N/2}(f)\right\}.
\eea

We now recall that
\be\label{pf1eqn3}
\mathcal{T}_N^\#(\mathcal{D})(\x)=\sigma_N(f)(\x)+\sum_{j=1}^M a_j\Phi_N(\x-\x_j), \qquad \x\in\TT^q.
\ee
Clearly, $\mathcal{T}_N^\#(\mathcal{D})\in\HH_N^q$, and $\mathcal{T}_N^\#(\mathcal{D})(\x_\ell)=y_\ell$, $\ell=1,\cdots,M$. This proves \eref{interpbasic}.

Moreover, for every $\x\in\TT^q$, \eref{pf1eqn3} and \eref{goodapprox} lead to
\bea\label{pf1eqn4}
|f(\x)-\mathcal{T}_N^\#(\mathcal{D})(\x)| &\le& |f(\x)-\sigma_N(f)(\x)|+\left\{\max_{1\le j\le M} |a_j|\right\}\sum_{j=1}^M|\Phi_N(\x-\x_j)|\nonumber\\
&\le& c\left\{E_{N/2}(f)+cN^{-q}\left\{\epsilon+E_{N/2}(f)\right\}\sum_{j=1}^M|\Phi_N(\x-\x_j)|\right\}.\nonumber\\
\eea
In view of \eref{phinsumest}, this leads to \eref{degreebasic}. \qed\\

\noindent\textbf{Proof of Theorem~\ref{net_feasibletheo}.}\\

This proof is very similar to that of Theorem~\ref{feasibletheo}. The condition \eref{netcond} and Proposition~\ref{matrix_inv_prop} ensure that the system of equations \eref{net_interp_eqns} has a unique solution satisfying \eref{invmatnorm}. 
The same argument as in \eref{pf1eqn4} then works with the operators as in Theorem~\ref{net_feasibletheo} replacing those in Theorem~\ref{feasibletheo}; we use \eref{relu_approx} in place of \eref{goodapprox}. \qed\\

\noindent\textbf{Proof of Theorem~\ref{constrtheo}.}\\

 In this proof, let   $\mathcal{T}^\#=T^\#_N(\mathcal{D})$ be as in Theorem~\ref{feasibletheo}, and
\be\label{pf3eqn1}
E= \epsilon+ E_{N/2}(f).
\ee
In view of \eref{degreebasic}, we may use Corollary~\ref{czipserfreudcor} to deduce that
\be\label{pf3eqn3}
\|f-\mathcal{T^\#}\|_{W^*}\le c\left\{NE + \sum_{j=1}^q E_N(D_j f)\right\}\le cN\left\{E+\frac{1}{N}\|f\|_{W^*}\right\}.
\ee
and in particular,
\be\label{pf3eqn4}
\|\mathcal{T^\#}\|_{W^*} \le cN\left\{E+\frac{1}{N}\|f\|_{W^*}\right\}.
\ee
Using \eref{interpbasic}, \eref{pf3eqn1} and \eref{pf3eqn4}, we obtain:
\bea\label{pf3eqn5}
\min_{T\in \HH_N}R_N(T)&\le& R_N(\mathcal{T}^\#)\le \max_{1\le j\le M}|y_j-\mathcal{T}^\#(\x_j)| +\frac{1}{N}\|\mathcal{T^\#}\|_{W^*}\nonumber\\
&=& \frac{1}{N}\|\mathcal{T^\#}\|_{W^*} \le c\left\{E+\frac{1}{N}\|f\|_{W^*}\right\}.
\eea
In view of Proposition~\ref{directtheoprop}, 
$$
E \le c\left\{\epsilon+\frac{1}{N}\|f\|_{W^*}\right\}.
$$
Together with \eref{pf3eqn5}, this
 proves \eref{regerrest}. 

Next, let $\x\in \TT^q$, and $\delta=\disp\min_{1\le j\le M}|\x-\x_j|=|\x-\x_\ell|$. For brevity, we write
 $$
  \tilde{E}= \epsilon+\frac{1}{N}\|f\|_{W^*}.
  $$ 
Using \eref{regerrest} and Corollary~\ref{czipserfreudcor}, we deduce that
\bea\label{pf3eqn6}
|f(\x)-T^*(\x)| &\le& |f(\x)-f(\x_\ell)| +|f(\x_\ell)-T^*(\x_\ell)| +|T^*(\x_\ell)-T^*(\x)|\nonumber\\
&\le& |\x-\x_\ell|\|f\|_{W^*} + c\tilde{E}+|\x-\x_\ell|\|T^*\|_{W^*}\le N\delta \frac{1}{N} \|f\|_{W^*} +c\tilde{E} +N\delta \tilde{E}\nonumber\\
 &\le& c(1+N\delta) \tilde{E}.
\eea
This proves \eref{generalerr}.
\qed

\begin{rem}\label{straightinterprem}
{\rm
If $\delta=\min_{1\le j\le M}|\x-\x_j|=|\x-\x_\ell|$, then we have
\begin{eqnarray*}
\lefteqn{\left|f(\x)-\frac{1}{\Phi_N(0)}\sum_{k=1}^M y_k\Phi_N(\x-\x_k)\right|}\\
&\le& |f(\x)-f(\x_\ell)|+\left|\frac{1}{\Phi_N(0)}\sum_{k=1}^M (y_k-f(\x_k))\Phi_N(\x-\x_k)\right| +\left|f(\x_\ell)- \frac{1}{\Phi_N(0)}\sum_{k=1}^M f(\x_k)\Phi_N(\x-\x_k)\right|\\
&\le& \delta\|f\|_{W^*}+\epsilon +\frac{1}{\Phi_N(0)}|f(\x_\ell)(\Phi_N(\x_\ell-\x_\ell)-\Phi_N(\x-\x_\ell))| +\frac{1}{\Phi_N(0)}\sum_{k\not=\ell}|f(\x_k)||\Phi_N(\x-\x_k)|\\
&\le&  \delta\|f\|_{W^*}+ c\epsilon +cN\delta\|f\| +\frac{1}{\Phi_N(0)}\sum_{k\not=\ell}|f(\x_k)||\Phi_N(\x-\x_k)|.
\end{eqnarray*}
Hence, if $2\pi\ge \delta \ge 1/N$, Lemma~\ref{regularlemma} used with $\nu$ as in the proof of Proposition~\ref{kernelopsummary_prop}(a), $d=\eta(\C)$, $r=\delta$ shows that 
$$
\frac{1}{\Phi_N(0)}\sum_{k\not=\ell}|f(\x_k)||\Phi_N(\x-\x_k)|\le c(N\delta)^{-S}(\delta+\eta(\C))^q\|f\| \le c(N\delta)^{-S}\|f\|.
$$
Therefore, if $1/N\le \delta\le 2\pi$,
\be\label{straightbd}
\left|f(\x)-\frac{1}{\Phi_N(0)}\sum_{k=1}^M y_k\Phi_N(\x-\x_k)\right|\le c\{\delta\|f\|_{W^*}+\epsilon +cN\delta\|f\|\}.
\ee
A similar bound can be proved  with the polynomial $\mathcal{L}_N(\mathcal{D})$.
\qed}
\end{rem}
\subsection{Proofs of the theorems in Section~\ref{deepsect}.}\label{deeppfs}
The induction argument given in the proof of Theorem~\ref{deepfeasibletheo} allows us to ``lift'' results about shallow networks to deep networks. We will refer to this argument as \emph{good propagation of error}.\\

\noindent\textbf{Proof of Theorem~\ref{deepfeasibletheo}.}\\
We observe that each node in $V\cup \mathbf{S}$ can be thought of as the sink node of an appropriate sub-DAG $\mathcal{G}_v$ of $\mathcal{G}$. Moreover, in the definitions such as \eref{gengfuncnormdef}, $\|f\|_{\XX,\mathcal{G}}\sim \sum_{v\in V\cup\mathbf{S}}\|f_v\|_{\XX,\mathcal{G}_v}$, and similarly for \eref{gaddegapproxdef}.  So, the statement of this theorem should be true also for each such sub-DAG.
Let $C=\max_{v\in V\cup\mathbf{S}}B(d(v))$ where $B(d(v))$ is the constant introduced in Proposition~\ref{kernelopsummary_prop} applied with $d(v)$ in place of $q$.

 Theorem~\ref{deepfeasibletheo} applied to a vertex in $\mathbf{S}$ is the same as Theorem~\ref{feasibletheo}.

Suppose the theorem is proved for every node of level up to $\ell$ for some $\ell\ge 0$, $v$ be a node at level $\ell+1$, and $u_1,\cdots,u_{d(v)}$ be its children. 
The sub-DAG with the outputs of these children as the  input to the sink node $v$ 
 is of course a shallow network, to which we may apply Theorem~\ref{feasibletheo}. 
Since  $N_v\ge C\eta(\C_v)^{-1}$ , Theorem~\ref{feasibletheo} implies that  there is $T_v^\#\in \HH_{N_v}^{d(v)}$ satisfying\footnote{Here, if $v=v^*$ is the sink node of the original DAG, the statement needs to be modified as follows: there is $T_{v^*}^\#\in\HH_N^{d(v^*)}$ that satisfies (thought of as a function of all the inputs to the network)
\be\label{pf2eqn2}
T_{v^*}^\#(\x_j)=T_{v^*}^\#((\x_j)_{v^*})=y_j, \qquad j=1,\cdots, M,
\ee
and (thought of as a function on $\TT^{d(v^*)}$)
\be\label{pf2eqn3}
\|f_{v^*}-T_{v^*}^\#\|_{d(v^*)} \le c\left\{\epsilon+E_{N_{v^*}/2}(d(v^*); f_{v^*})\right\}.
\ee}
\be\label{pf2eqn1}
T_v^\#(\x)=f_v(\x), \qquad \x\in \C_v, \qquad \|f_v-T_v^\#\|_{d(v)}\le cE_{N_v/2}(d(v); f_v).
\ee

Our induction hypothesis shows that the analogues of the estimates \eref{pf2eqn1} and \eref{pf2eqn2} are true for each of the children $u_1,\cdots, u_{d(v)}$.
Necessarily, $v\in V$, and 
 so, each $f_v$ is Lipschitz continuous with Lipschitz constant $\le L$.
Using triangle inequality, we deduce  that
\begin{eqnarray*}
\lefteqn{|f_v(f_{u_1}(\x_{u_1}),\cdots,f_{u_{d(v)}}(\x_{u_{d(v)}}))-T^\#_v(T^\#_{u_1}(\x_{u_1}),\cdots,T^\#_{u_{d(v)}}(\x_{u_{d(v)}}))|}\\
&\le& |f_v(T^\#_{u_1}(\x_{u_1}),\cdots,T^\#_{u_{d(v)}}(\x_{u_{d(v)}}))-T^\#_v(T^\#_{u_1}(\x_{u_1}),\cdots,T^\#_{u_{d(v)}}(\x_{u_{d(v)}}))|\\
&&\qquad + |f_v(f_{u_1}(\x_{u_1}),\cdots,f_{u_{d(v)}}(\x_{u_{d(v)}}))-f_v(T^\#_{u_1}(\x_{u_1}),\cdots,T^\#_{u_{d(v)}}(\x_{u_{d(v)}}))|\\
&\le& \sup_{\y\in\TT^{d(v)}}|f_v(\y)-T^\#_v(\y)| + L\left\|(f_{u_1}(\x_{u_1}),\cdots,f_{u_{d(v)}}(\x_{u_{d(v)}}))-(T^\#_{u_1}(\x_{u_1}),\cdots,T^\#_{u_{d(v)}}(\x_{u_{d(v)}}) )\right\|_1\\
&\le& \|f_v-T^\#_v\|_{d(v)} +L\sum_{j=1}^{d(v)}\|f_{u_j}-T^\#_{u_j}\|_{C^*, \mathcal{G}_{u_j}}.
\end{eqnarray*}
In view of our dual interpretation of the constituent functions as functions of their immediate input as well as the input seen by these functions and the induction hypothesis, this shows that  Theorem~\ref{deepfeasibletheo} is valid for the sub-DAG with $v$ as the sink node.
We define
\be\label{deep_tshar_def} 
T_{\mathbf{N}}^\#(\mathcal{D})=\{T^\#_v\}\in \mathcal{G}-\HH_\mathbf{N}.
\ee
 The equation \eref{pf2eqn2} is the same as \eref{deepinterpbasic}. 
 The estimate \eref{deepdegreebasic} follows by induction as just explained.
\qed\\

\noindent\textbf{Proof of Theorem~\ref{deepconstrtheo}.}\\

This proof is essentially the same as that of Theorem~\ref{constrtheo}. 
We point out some considerations required in the details which are different.
Since each $f_v\in W^*(\TT^{d(v)})$, the simultaneous approximation theorem Corollary~\ref{czipserfreudcor} holds for each $f_v$. 
We may assume that the perturbation exists only at the sink node, and the rest of the data is exact. Finally, we use the good propagation of error to obtain \eref{deepregerrest}. 
The proof of \eref{deepgeneralerr} involves an argument similar to \eref{pf3eqn6}. The middle term on the first line of that chain of inequalities gives the first term on the right hand side in  \eref{deepgeneralerr}. The other terms can be estimated as in \eref{pf3eqn6}, except that the chain rule of differentiation needs to be used several times to get to the level of the source nodes. This is facilitated by our extra assumption that $\max_{v\in V}\|f_v\|_{W^*(\TT^{(d(v)})}\le L$. 
\qed

\bhag{Conclusions and open problems}
We have explored the puzzle that deep networks (and sometimes also shallow ones) do not exhibit over-fitting even though the number of parameters is very large and the training error is reduced to zero. We have initiated a rigorous study of this phenomenon from the point of view of function approximation, giving estimates on how many parameters are needed to exhibit a zero or good training error, which is also compatible with the generalization error. Our estimates are given in terms of the data characteristics and the smoothness of the target function. 

One obvious problem is to reduce the number of parameters to be trained in the regularization functional \eref{regularizationdef}.
 Expressing the trigonometric polynomials as linear combinations of the  translates  $\Phi_N(\circ-2\pi\j/N)$ ensures that the resulting solution will have 
 coefficients that are small away from the training data. 
  It is therefore reasonable to take only some of these translates that are close to the training data.
However, it is not clear that the theory will work with the space defined by the span of only those translates which are kept. Also, there may be some numerical instability problems with this basis. 

A deeper and wider area of theoretical investigation is the following. Theorems~\ref{constrtheo} and \ref{deepconstrtheo} suggest that the approximation error given by the
trigonometric polynomials constructed there gives an estimate on the support of the marginal distribution of the training and test data.
In \cite{spie10}, the approximation errors of the convergent bounded interpolatory polynomials constructed in \cite{bdint} were used for
texture detection and segmentation of images. 
What would be the analogues for understanding the nature of the data using the approximation errors obtained here?


\end{document}